\documentclass[twoside]{article}

\usepackage[utf8]{inputenc} 
\usepackage[T1]{fontenc}    
\usepackage[round]{natbib}
\usepackage[english]{babel}
\usepackage{hyperref}
\usepackage{amsmath,amsthm,amssymb}
\usepackage{subcaption}

\author{anonymous}
\usepackage{xcolor}
\definecolor{linkcolor}{RGB}{83,83,182}
\definecolor{citecolor}{RGB}{128,0,128}
\hypersetup{
    colorlinks=true,
    citecolor=citecolor,
    linkcolor=linkcolor,
    urlcolor=linkcolor
}

\usepackage[capitalise,nameinlink]{cleveref} 
\usepackage{booktabs}  
\usepackage{quentin_shortcuts}
\usepackage{multicol}

\usepackage[accepted]{aistats2021}
\usepackage{siunitx}
\usepackage{blindtext}
\usepackage{float} 

\newcommand{\GD}{\text{GD}}
\usepackage{algorithm}
\usepackage{algorithmic}
\usepackage[titlenumbered,linesnumbered,ruled,noend,algo2e]{algorithm2e}

\SetCommentSty{mycommfont}
\SetEndCharOfAlgoLine{}

\begin{document}

\twocolumn[
  \aistatstitle{Anderson acceleration of coordinate descent}
  \aistatsauthor{Quentin Bertrand \And Mathurin Massias}
\aistatsaddress{
Universit\'e Paris-Saclay \\ Inria, CEA  \\ Palaiseau, France  \And
MaLGa \\ DIBRIS \\ University of Genova}
]

\begin{abstract}

Acceleration of first order methods is mainly obtained via inertia à la Nesterov, or via nonlinear extrapolation.
The latter has known a recent surge of interest, with successful applications to gradient and proximal gradient techniques.
On multiple Machine Learning problems, coordinate descent achieves performance significantly superior to full-gradient methods.
Speeding up coordinate descent in practice is not easy:
inertially accelerated versions of coordinate descent are theoretically accelerated, but might not always lead to practical speed-ups.
We propose an accelerated version of coordinate descent using extrapolation, showing considerable speed up in practice, compared to inertial accelerated coordinate descent and extrapolated (proximal) gradient descent.
Experiments on least squares, Lasso, elastic net and logistic regression validate the approach.

\end{abstract}

%
\section{Introduction}
Gradient descent is the workhorse of modern convex optimization \citep{Nesterov04,Beck17}.
For composite problems, proximal gradient descent retains the nice properties enjoyed by the latter.
In both techniques, inertial acceleration achieves accelerated convergence rates \citep{Nesterov83,Beck_Teboulle09}.

Coordinate descent is a variant of gradient descent, which updates the iterates one coordinate at a time \citep{tseng2009coordinate,Friedman_Hastie_Tibshirani10}.
Proximal coordinate descent has been applied to numerous Machine Learning problems \citep{Shalev-Shwartz_Zhang13,Wright_2015,Shi_Tu_Xu_Yin_2016}, in particular the Lasso \citep{Tibshirani96}, elastic net \citep{Zou_Hastie05} or sparse logistic regression \citep{Ng04}.
It is used in preeminent packages such as scikit-learn \citep{Pedregosa_etal11}, glmnet \citep{Friedman_Hastie_Tibshirani2009},
libsvm \citep{Fan_Chang_Hsieh_Wang_Lin08} or lightning \citep{Blondel_Pedregosa2016}.
On the theoretical side, inertial accelerated versions of coordinate descent \citep{Nesterov2012,Lin_Lu_Xiao_2014,Fercoq_Richtarik2015} achieve accelerated rates.
Note that usual lower bounds \citep[Sec. 2.1.2]{Nesterov04} are derived for methods with iterates lying in the span of previous gradients, which is not the case for coordinate descent.
However there also exists similar lower bounds for cyclic coordinate descent \citep{Sun_Ye2019}.

\begin{figure}[tb]
  \centering
  \includegraphics[width= \linewidth]{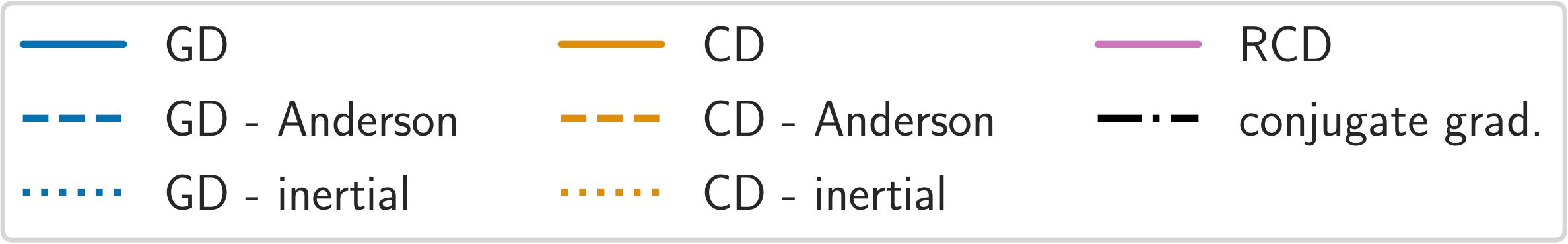}
  \includegraphics[width=\linewidth]{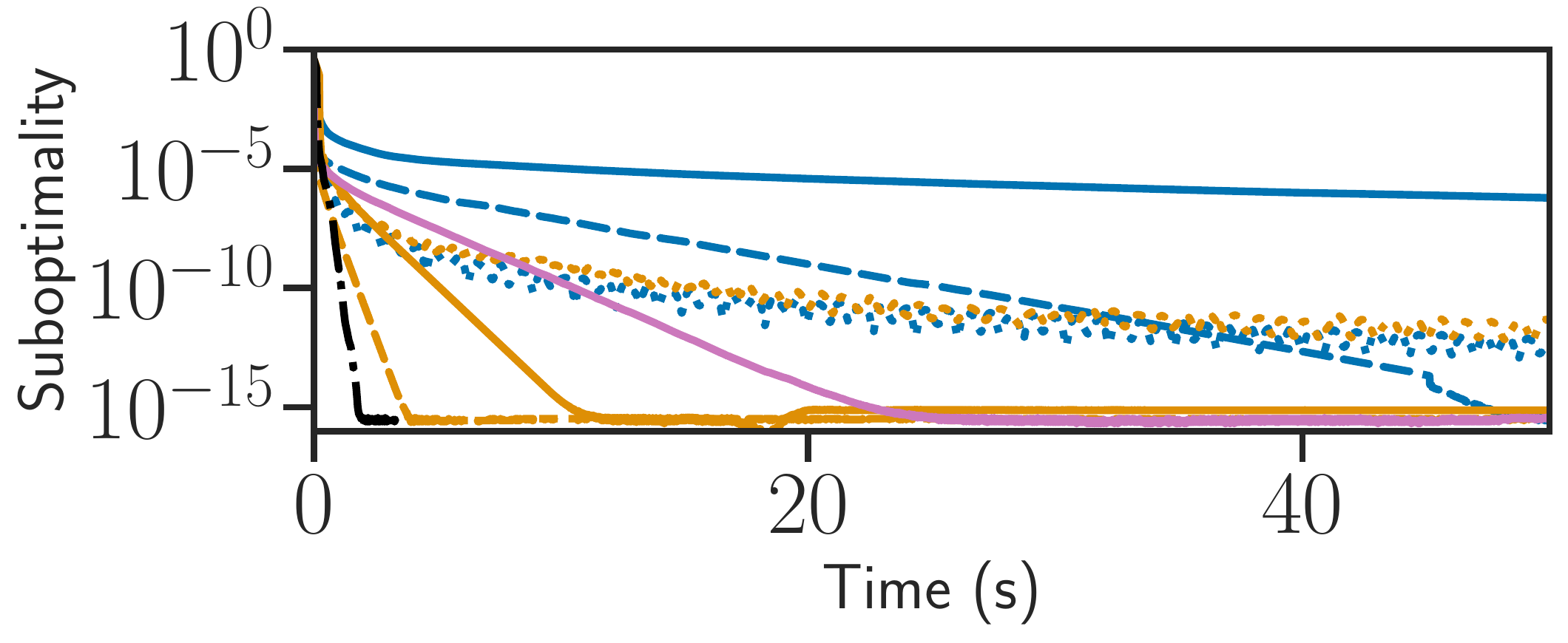}
  \caption{Suboptimality along time for a quadratic problem on the 5000 first features of the \emph{rcv1} dataset.
  GD: gradient descent,
  (R)CD: (randomized) coordinate descent.
  }
  \label{fig:intro}
\end{figure}
\sloppy
To obtain accelerated rates, Anderson extrapolation \citep{Anderson65} is an alternative to inertia: it provides acceleration by exploiting  the iterates' structure.
This procedure has been known for a long time, under various names and variants \citep{Wynn1962,Eddy1979,Smith_Ford_Sidi1987}, see \citet{Sidi17,Brezinski2018} for reviews.
Anderson acceleration enjoys accelerated rates on quadratics \citep{Golub_Varga1961},
but theoretical guarantees in the nonquadratic case are weaker \citep{Scieur_dAspremont_Bach2016}.
Interestingly, numerical performances still show significant improvements on nonquadratic objectives.
Anderson acceleration has been adapted to various algorithms such as Douglas-Rachford \citep{Fu2019}, ADMM \citep{Poon_Liang2019} or proximal gradient descent \citep{Zhang_ODonoghe_Boyd2018,Mai_Johansson_19,Poon_Liang2020}.
Among main benefits, the practical version of Anderson acceleration is memory efficient, easy to implement, line search free, has a low cost per iteration and does not require knowledge of the strong convexity constant.
Finally, it introduces a single additional parameter, which often does not require tuning (see \Cref{sub:parameter_setting}).

In this work:
\begin{itemize}
  \item  We propose an Anderson acceleration scheme for coordinate descent, which, as visible on \Cref{fig:intro}, outperforms inertial and extrapolated gradient descent, as well as inertial and randomized coordinate descent.
  \item The acceleration is obtained even though the iteration matrix is not symmetric, a notable problem in the analysis of Anderson extrapolation.
  \item We empirically highlight that the proposed acceleration technique can generalize in the non-quadratic case (\Cref{alg:anderson_cd}) and significantly improve proximal coordinate descent algorithms (\Cref{sec:expes}), which are state-of-the-art first order methods on the considered problems.
\end{itemize}
\begin{figure*}[t]
  \begin{minipage}[t]{0.47\linewidth}
      \begin{algorithm}[H]
        \SetKwInOut{Init}{init}
        \Init{$x^{(0)} \in \bbR^p$}
        \caption{Offline Anderson extrapolation}\label{alg:offline}
        \For{$k = 1, \ldots$}{
            $x^{(k)} = T x^{(k - 1)} + b$ \tcp*[r]{regular linear iteration}

            $U = [x^{(1)} - x^{(0)}, \ldots, x^{(k)} - x^{(k - 1)}]$

            $c = (U^\top U)^{-1} \mathbf{1}_k / \mathbf{1}_k^\top (U^\top U)^{-1} \mathbf{1}_k \in \bbR^k$

            $x_{\mathrm{e-off}}^{(k)} = \sum_{i=1}^k c_i x^{(i)}$ \tcp*[r]{does not affect $x^{(k)}$}
        }
      \Return{$x_{\mathrm{e-off}}^{(k)}$}
      \end{algorithm}
      \begin{subfigure}{1\textwidth}
        \centering
        \includegraphics[width=0.6\linewidth]{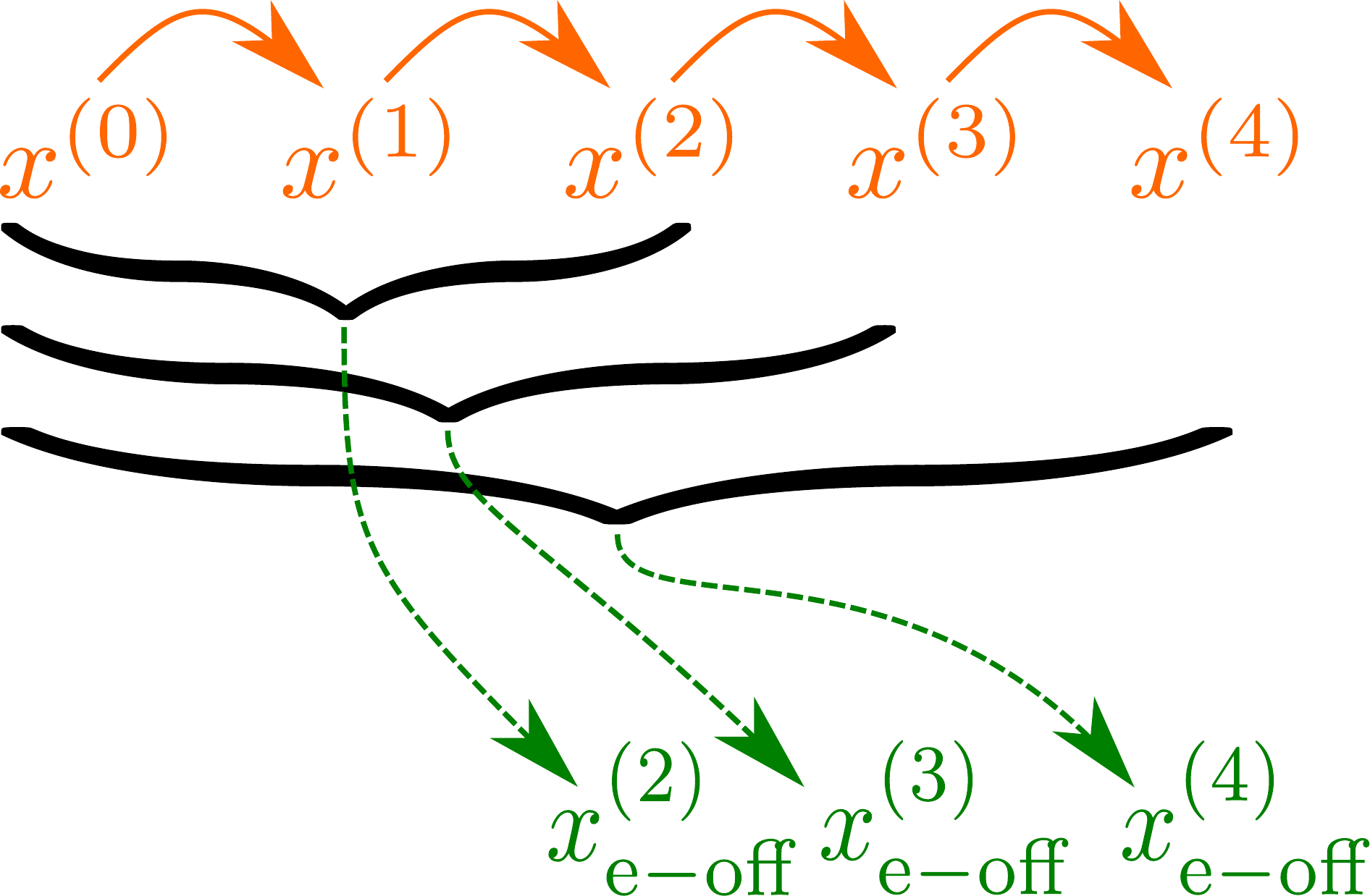}
        \caption{Offline.}
        \label{fig:aa_offline}
    \end{subfigure}
  \end{minipage}
  \hfill
  \begin{minipage}[t]{0.47\linewidth}
    \begin{algorithm}[H]
    \caption{Online Anderson extrapolation}\label{alg:online}
    \SetKwInOut{Init}{init}
    \setcounter{AlgoLine}{0}
    \Init{$x^{(0)} \in \bbR^p$}

    \For{$k = 1, \ldots$}{

        $x^{(k)} = T x^{(k - 1)} + b$\tcp*[r]{regular iteration}

        \If{$k = 0 \quad \mathrm{mod} \, K$}{
            $U = [x^{(k - K + 1)} - x^{(k - K)}, \ldots, x^{(k)} - x^{(k - 1)}]$

            $c = (U^\top U)^{-1} \mathbf{1}_K / \mathbf{1}_K^\top (U^\top U)^{-1} \mathbf{1}_K \in \bbR^K$

            $x_{\mathrm{e-on}}^{(k)} = \sum_{i=1}^K c_i x^{(k - K + i)}$

            $x^{(k)} = x_{\mathrm{e-on}}^{(k)}$\tcp*[r]{base sequence changes}
        }
    }
    \Return{$x^{(k)}$}
  \end{algorithm}
  \begin{subfigure}{1\textwidth}
      \centering
      \includegraphics[width=0.6\linewidth]{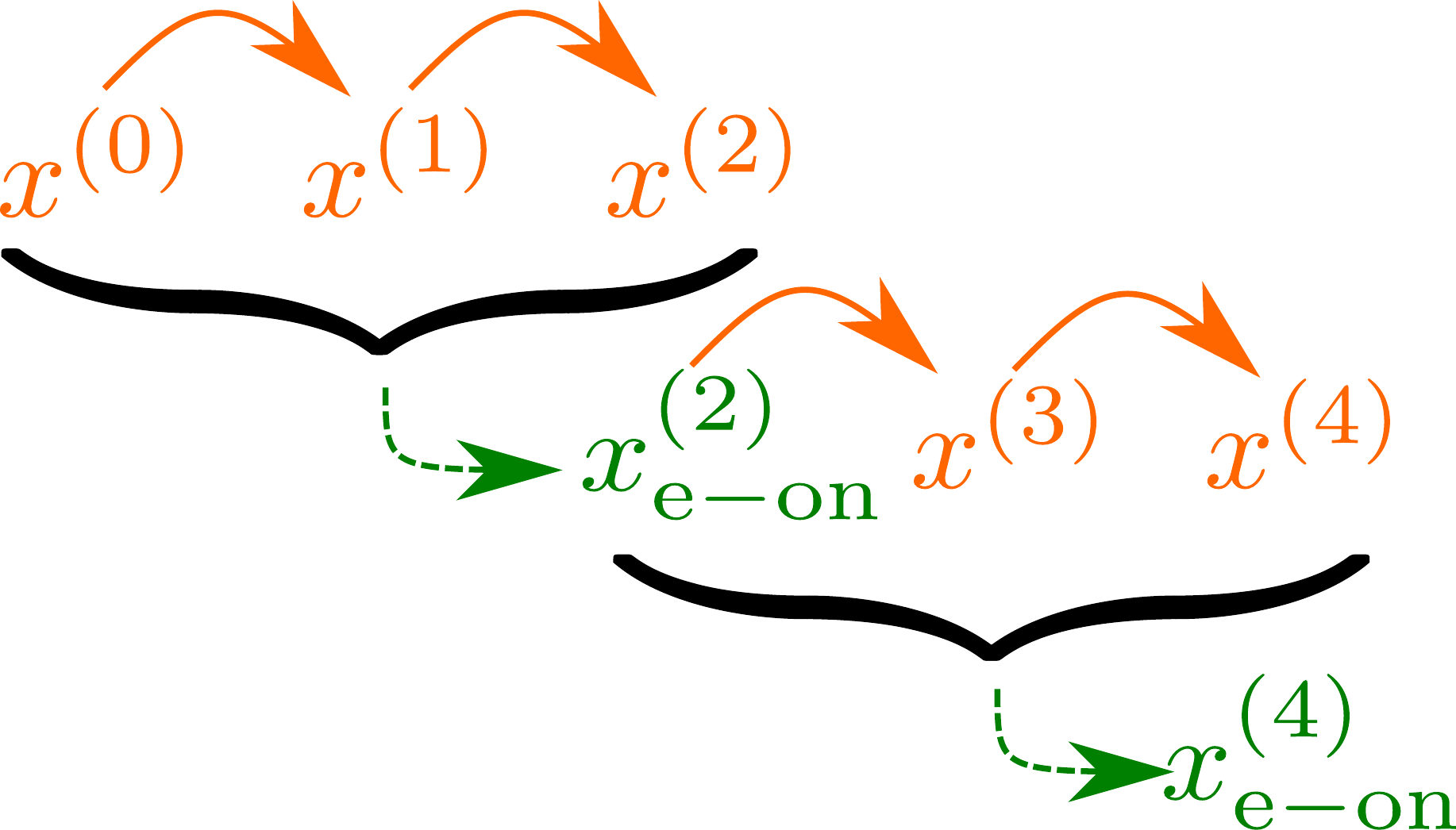}
      \caption{Online.}
      \label{fig:aa_online}
  \end{subfigure}
  \end{minipage}
  \caption{Illustrations of offline (left) and online (right) Anderson extrapolation for $K=2$.}
  \label{fig:aa_offline_online}
\end{figure*}

\paragraph{Notation}
The $j$-th line of the matrix $A$ is $A_{j:}$ and its $j$-th column is $A_{:j}$.
The canonical basis vectors of $\bbR^p$ are $e_j$.
The vector of size $K$ with all one entries is $\mathbf{1}_K$.
The spectral radius of the matrix $A$, $\rho(A)$, is the largest eigenvalue modulus of $A$.
The set of $p$ by $p$ symmetric positive semidefinite matrices is $\bbS_{+}^p$.
The condition number $\kappa(A)$ of a matrix $A$ is its largest singular value divided by its smallest.
A positive definite matrix $A$ induces the norm $\normin{x}_A = \sqrt{x^\top A x}$.
The proximity operator of the function $g$ is $\prox_g(x) = \argmin_y g(y) + \frac{1}{2} \norm{x - y}^2$.

%
\section{Anderson extrapolation}
\label{sec:anderson}

\subsection{Background}
Anderson extrapolation is designed to accelerate  the convergence of sequences based on fixed point linear iterations, that is:
\begin{equation}
    x^{(k + 1)} = T x^{(k)} + b \enspace ,
\end{equation}
where the \emph{iteration matrix} $T \in \bbR^{p \times p}$ has spectral radius $\rho(T) < 1$.
There exist two variants: offline and online, which we recall briefly.

\emph{Offline} extrapolation (\Cref{alg:offline,fig:aa_offline}), at iteration $k$, looks for a fixed point as an affine combination of the $k$ first iterates: $x_\text{e-off}^{(k)} = \sum_{1}^k c_i^{(k)} x^{(i - 1)}$, and solves for the coefficients $c^{(k)} \in \bbR^k$ as follows:
\begin{align}\label{pb:anderson}
    c^{(k)}
    &= \argmin_{\sum_1^k c_i = 1} \normin{\textstyle\sum\nolimits_{1}^k c_ix^{(i - 1)} - T\textstyle\sum\nolimits_{1}^k c_i x^{(i - 1)} - b}^2
    \nonumber\\
    &= \argmin_{\sum_1^k c_i = 1}  \normin{\textstyle\sum\nolimits_1^k c_i \big(x^{(i)} - x^{(i - 1)}\big)}^2
    \nonumber \\
    &= (U^\top U)^{-1} \mathbf{1}_k / \mathbf{1}_k^\top (U^\top U)^{-1} \mathbf{1}_k
    \enspace ,
\end{align}
where $U = [x^{(1)} - x^{(0)}, \ldots, x^{(k)} - x^{(k - 1)}] \in \bbR^{p \times k}$ (and hence the objective rewrites $\norm{Uc}^2$).
In practice, since $x^{(k)}$ is available when $c^{(k)}$ is computed, one uses $x_\text{e}^{(k)} = \sum_{1}^k c_i^{(k)} x^{(i)}$ instead of $\sum_{1}^k c_i^{(k)} x^{(i - 1)}$.
The motivation for introducing the coefficients $c^{(k)}$ is discussed in more depth after Prop. 6 in \citet{Massias_Vaiter_Salmon_Gramfort2019}, and details about the closed-form solution can be found in \citet[Lem. 2.4]{Scieur_dAspremont_Bach2016}.
In offline acceleration, more and more base iterates are used to produce the extrapolated point, but the extrapolation sequence does not affect the base sequence.
This may not scale well since it requires solving larger and larger linear systems.

A more practical variant is the \emph{online} version  (\Cref{alg:online,fig:aa_online}), considered in this paper.
The number of points to be extrapolated is fixed to $K$; $x^{(1)}, \ldots, x^{(K)}$ are computed normally with the fixed point iterations, but $x_\mathrm{e}^{(K)}$ is computed by extrapolating the iterates from $x^{(1)}$ to $x^{(K)}$, and $x^{(K)}$ is taken equal to $x_\text{e}^{(K)}$.
$K$ normal iterates are then computed from $x^{(K + 1)}$ to $x^{(2K)}$ then extrapolation is performed on these last $K$ iterates, etc.

\begin{remark}
    The proposed \emph{online} version (\Cref{alg:online,fig:aa_online}) slightly differs from the \emph{online} algorithms in \citet{Walker_Ni2011,Mai_Johansson_19}.
    For computational purposes the extrapolation is performed every $K$ step, see \Cref{sec:expes} for details.
\end{remark}

As we recall below, results on Anderson acceleration mainly concern fixed-point iterations with {symmetric} iteration matrices $T$, and results concerning non-symmetric iteration matrices are weaker \citep{Bollapragada2018}.
\citet[Thm 6.4]{Poon_Liang2020} do not assume that $T$ is symmetric, but only diagonalizable, which is still a strong requirement.
\begin{proposition}[{Symmetric $T$, \citealt{Scieur2019}}]\label{prop:acc_rate_sym}
    Let the iteration matrix $T$ be symmetric semi-definite positive, with spectral radius $\rho = \rho(T) < 1$.
    Let $x^*$ be the limit of the sequence $(x^{(k)})$.
    Let $ \zeta = (1 - \sqrt{1 - \rho}) / (1 + \sqrt{1- \rho})$.
    Then the iterates of offline Anderson acceleration satisfy, with $B=(\Id - T)^2$:
    \begin{equation}
        \normin{x_{\text{e-off}}^{(k)} - x^*}_{B}
        \leq
        \tfrac{2\zeta^{k-1}}{1 + \zeta^{2(k-1)}} \normin{x^{(0)} - x^*}_{B} \enspace,
    \end{equation}
    and thus those of online extrapolation satisfy:
    \begin{equation}
        \normin{x_{\text{e-on}}^{(k)} - x^*}_{B}
        \leq
        \Big(
            \tfrac{2\zeta^{K-1}}{1 + \zeta^{2(K-1)}}
        \Big)^{k / K} \normin{x^{(0)} - x^*}_{B} \enspace.
    \end{equation}
\end{proposition}
\citet{Scieur_dAspremont_Bach2016} showed that the offline version in \Cref{prop:acc_rate_sym} matches the accelerated rate of the conjugate gradient \citep{Hestenes_Stiefel1952}.
As it states, gradient descent can be accelerated by Anderson extrapolation on quadratics.

%
%
\paragraph{Application to least squares}
%
The canonical application of Anderson extrapolation is gradient descent on least squares.
Consider a quadratic problem, with $b \in \bbR^p$, $H \in \bbS_{++}^{p}$ such that $0 \prec H \preceq L$ and $L > 0$:
\begin{problem}\label{pb:quadratic}
    x^* = \argmin_{x \in \bbR^p}
    \frac{1}{2} x^\top H x + \langle b, x \rangle
    \enspace .
\end{problem}
%
A typical instance is overdetermined least squares with full-column rank design matrix $A \in \bbR^{n \times p}$, and observations $y \in \bbR^n$, such that $H = A^\top A$ and $b=-A^\top y$.
On \Cref{pb:quadratic} gradient descent with step size $1/L$ reads:
\begin{equation}\label{eq:gd_linear}
    x^{(k+1)}
    =
    \underbrace{
        \left(\Id_p - \tfrac{1}{L} H\right)
    }_{
        T^{\text{GD}} \in \bbS_{+}^p
        } x^{(k)}
    + (\underbrace{-b/L}_{b^\text{GD}})\enspace.
\end{equation}
Because they have this linear structure, iterates of gradient descent can benefit from Anderson acceleration, observing that the fixed point of $x \mapsto T^\GD x + b^\GD$ solves \eqref{pb:quadratic},
with $T^{\text{GD}} \in \bbS_{+}^p$.
Anderson acceleration of gradient descent has therefore been well-studied beyond the scope of Machine Learning \citep{Pulay1980,Eyert1996}.
However, on many Machine Learning problems, coordinate descent achieves far superior performance, and it is interesting to determine whether or not it can also benefit from Anderson extrapolation.
%
\subsection{Linear iterations of coordinate descent}
To apply Anderson acceleration to coordinate descent, we need to show that its iterates satisfy linear iterations as in \eqref{eq:gd_linear}.
An epoch of cyclic coordinate descent for \Cref{pb:quadratic} consists in updating the vector $x$ one coordinate at a time, sequentially, i.e. for $j = 1, \dots, p$:
\begin{align}\label{eq:cd_update}
    x_j \leftarrow x_j - \frac{1}{H_{jj}} (H_{j:} x + b_j) \enspace,
\end{align}
which can be rewritten, for $j = 1, \dots, p$:
\begin{align}
    x \leftarrow \left (
        \Id_p - \frac{e_j e_j^\top}{H_{jj}} H \right ) x
    - \frac{b_j}{H_{jj}} e_j \enspace.
\end{align}
Thus, for primal iterates, as observed by \citet[Sec. A.3]{bertrand2020implicit}, one full pass (updating coordinates from 1 to $p$) leads to a linear iteration:
\begin{align}
    x^{(k+1)}
    &=
    T^{\text{CD}} x^{(k)} + b^{\text{CD}}
    \enspace ,
    \label{eq:cd}
\end{align}
with $T^{\text{CD}} = \Big( \Id_p - \frac{e_p e_p^\top}{ H_{pp} } H   \Big)
        \dots \Big( \Id_p - \tfrac{e_1 e_1^\top}{ H_{11}} H   \Big)$.
Note that in the case of coordinate descent we write $x^{(k)}$ for the iterates after one pass of coordinate descent on all features, and not after each update \eqref{eq:cd_update}.
The iterates of coordinate therefore also have a fixed-point structure, but contrary to gradient descent, their iteration matrix $T^{\text{CD}}$ is not symmetric, which we address in \Cref{sub:non_symmetric}.

%
\subsection{Anderson extrapolation for nonsymmetric iteration matrices}\label{sub:non_symmetric}
%
%
%
Even on quadratics, Anderson acceleration with non-symmetric iteration matrices is less developed, and the only results concerning its theoretical acceleration are recent and weaker than in the symmetric case.
\begin{proposition}[{\citealt[Thm 2.2]{Bollapragada2018}}]
\label{prop:acc_rate_nonsym}
When $T$ is not symmetric, and $\rho(T) < 1$,
\begin{align*}
    &\normin{x_\text{e-off}^{(k)} - T x_\text{e-off}^{(k)} - b} \leq  \nonumber\\
    &\quad \norm{\Id - \rho(T - \Id)}_2
    \normin{P^*(T)(x^{(1)} - x^{(0)})} \enspace,
\end{align*}
\sloppy
where the unavailable polynomial $P^*$ minimizes $\norm{P(T)(x^{(1)} - x^{(0)})}$ amongst all polynomials $P$ of degree exactly $k - 1$ whose coefficients sum to 1.
\end{proposition}
The quality of the bound (in particular, its eventual convergence to 0) crucially depends on $\normin{P(T)}$.
Using the Crouzeix conjecture \citep{Crouzeix2004} \citet{Bollapragada2018} managed to bound $\normin{P(T)}$, with $P$ a polynomial:
\begin{align}
    \normin{P(T)} \leq c \max_{z \in W(T)} |P(z)|
    \enspace ,
\end{align}
with $c \geq 2$ \citep{Crouzeix2007,Crouzeix_Palencia2017}, and $W(T)$ the numerical range:
\begin{align}
    W(T) \eqdef \{ x^* T x : \normin{x}_2 = 1, x \in \bbC^p \}
    \enspace .
\end{align}
\sloppy
Since there is no general formula for this bound, \citet{Bollapragada2018} used numerical bounds on  $W(T^q)$ to ensure convergence.
\Cref{fig:rayleigh} displays the numerical range $W(T^q)$ in the complex plane for $q \in \{1, 128, 256, 512 \}$.
In order to be able to apply the theoretical result from \citet{Bollapragada2018}, one must chose $q$ such that the point $(1, 0)$ is not contained in $W(T^q)$, and extrapolate $x^{(0)}, x^{(q)}, x^{(2q)}, \ldots$
One can see on \Cref{fig:rayleigh} that large values of $q$ are needed, unusable in practice: $q=512$ is greater than the number of iterations needed to converge on some problems.
Moreover, Anderson acceleration seems to provide speed up on coordinate descent even with $q=1$ as we perform, which highlights the need for refined bounds for Anderson acceleration on nonsymmetric matrices.

We propose two means to fix this lack of theoretical results: to modify the algorithm in order to have a more amenable iteration matrix (\Cref{sub:pseudosymcd}), or to perform a simple cost function decrease check (\Cref{sub:algorithm}).
%
%
\begin{figure}[tb]
    \centering
    \includegraphics[height=28px]{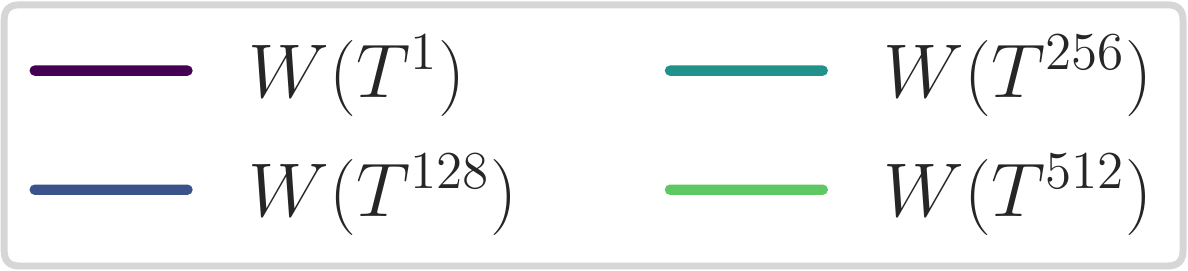}
    \includegraphics[width=\linewidth]{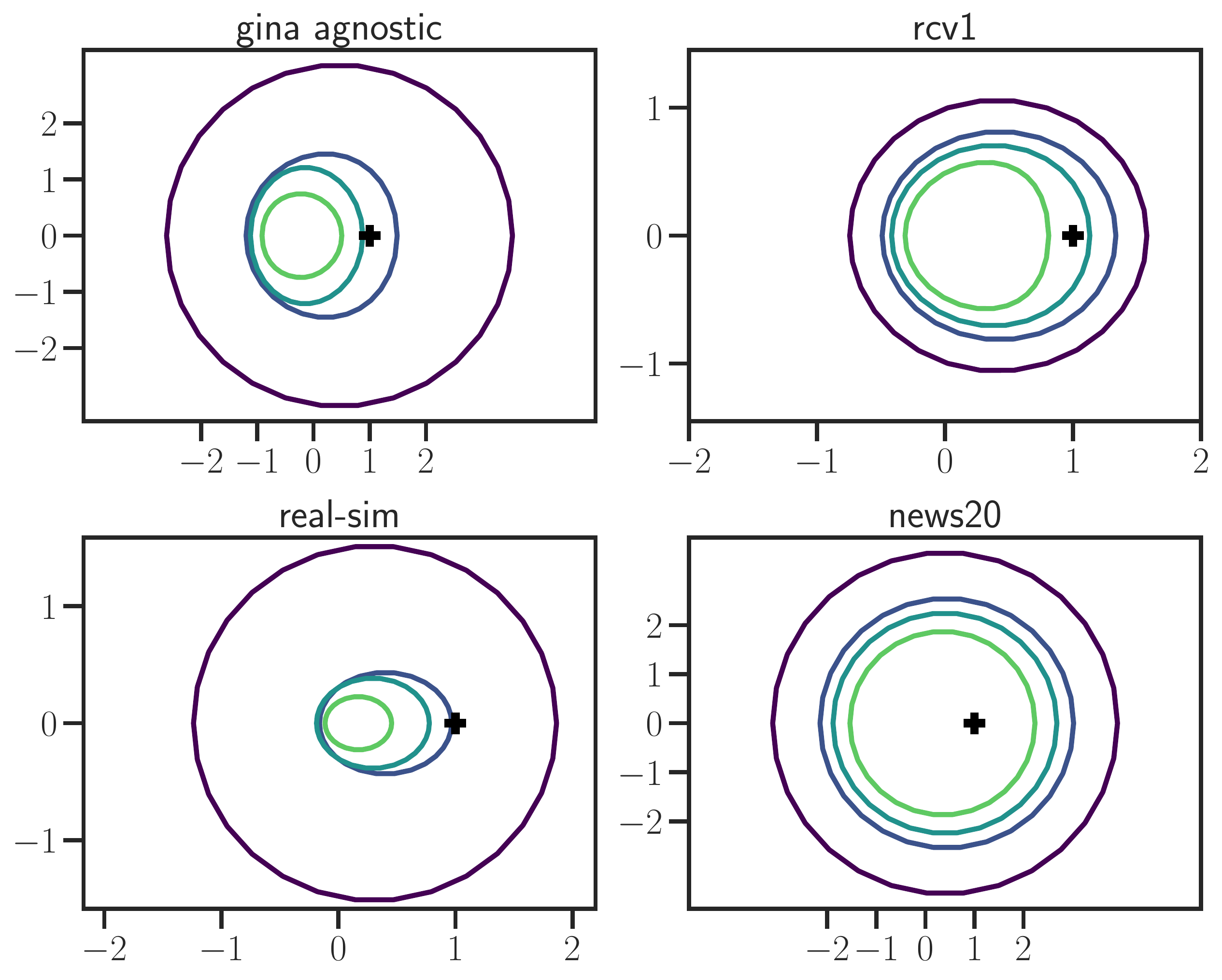}
\caption{
Numerical range of $T^q$ as $q$ varies; $T$ is the iteration matrix of Ridge regression problems with conditioning $\kappa= 10^3$, on 4 datasets.
The black cross marks the $(1, 0)$ point, which should lie outside the range for the theoretical bound to be useful.}
\label{fig:rayleigh}
\end{figure}

\subsection{Pseudo-symmetrization of $T$}
\label{sub:pseudosymcd}
%
%
A first idea to make coordinate descent theoretically amenable to extrapolation is to perform updates of coefficients from indices $1$ to $p$, followed by a reversed pass from $p$ to $1$.
This leads to an iteration matrix which is not symmetric either but friendlier: it writes
\begin{equation}\label{eq:T_pseudosym}
    T^{\text{CD-sym}} \eqdef H^{-1/2} S H^{1/2}
    \enspace ,
\end{equation}
with
\begin{align}\label{eq:S_pseudosym}
    S &= \Big( \Id_p - H^{1/2}\tfrac{e_1 e_1^\top}{ H_{11}} H^{\frac{1}{2}}   \Big)
    \times
    \dots
    \times
    \Big( \Id_p - H^{\frac{1}{2}} \tfrac{e_p e_p^\top}{ H_{pp} } H^{\frac{1}{2}} \Big)
    \nonumber
    \\
    &\quad \times
    \Big( \Id_p - H^{\frac{1}{2}} \tfrac{e_p e_p^\top}{ H_{pp} } H^{\frac{1}{2}}   \Big)
    \times
    \dots
    \times
    \Big( \Id_p - H^{\frac{1}{2}} \tfrac{e_1 e_1^\top}{ H_{11}} H^{\frac{1}{2}} \Big)
    \enspace .
\end{align}
$S$ is symmetric, thus, $S$ and $T$ (which has the same eigenvalues as $S$), are diagonalisable with real eigenvalues.
We call these iterations pseudo-symmetric, and show that this structure allows to preserve the guarantees of Anderson extrapolation.
\begin{proposition}[{Pseudosym. $T=H^{-1/2} S H^{1/2}$}]
    \label{prop:acc_rate_cdsym}
    Let $T$ be the iteration matrix of pseudo-symmetric coordinate descent: $T=H^{-1/2} S H^{1/2}$, with $S$ the symmetric positive semidefinite matrix of \eqref{eq:T_pseudosym}.
    Let $x^*$ be the limit of the sequence $(x^{(k)})$.
    Let $ \zeta = (1 - \sqrt{1 - \rho}) / (1 + \sqrt{1- \rho})$.
    Then $\rho = \rho(T) = \rho(S) < 1$ and the iterates of offline Anderson acceleration satisfy, with
    $B=( T - \Id)^\top ( T - \Id)$:
    \begin{equation}
        \normin{x_{\text{e-off}}^{(k)} - x^*}_{B}
        \leq
        \sqrt{\kappa(H)}
        \tfrac{2\zeta^{k-1}}{1 + \zeta^{2(k-1)}}
        \normin{x^{(0)} - x^*}_B
        \enspace,
    \end{equation}
    and thus those of online extrapolation satisfy:
    \begin{equation}
        \normin{x_{\text{e-on}}^{(k)} - x^*}_B
        \leq
        \Big(
            \sqrt{\kappa(H)}
            \tfrac{2\zeta^{K-1}}{1 + \zeta^{2(K-1)}} \Big)^{k / K}
        \normin{x^{(0)} - x^*}_B
        \enspace.
    \end{equation}
\end{proposition}
Proof of \Cref{prop:acc_rate_cdsym} can be found in \Cref{app:proofs}.
\Cref{prop:acc_rate_cdsym} shows 
accelerated convergence rates for the offline Anderson acceleration, but a $\sqrt{\kappa(H)}$ appears in the rate of the online Anderson acceleration, meaning that $K$ must be large enough that $\zeta^K$ mitigates this effect.
This factor however seems like a theoretical artefact of the proof, since we observed significant speed up of the online Anderson acceleration, even with bad conditioning of $H$ (see \Cref{fig:cd_sym_ols_rcv1}).
%
\begin{figure}[tb]
    \centering
    \includegraphics[height=23px]{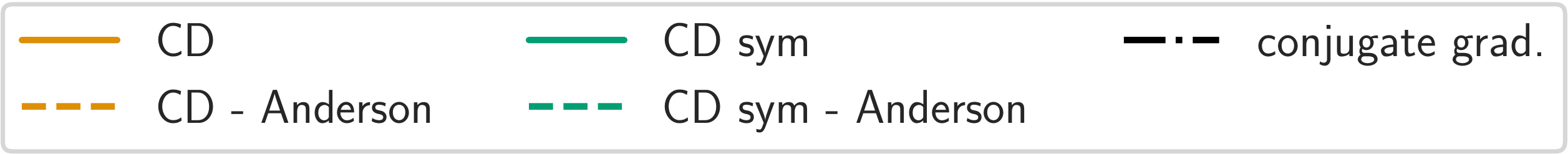}
    \includegraphics[width=\linewidth]{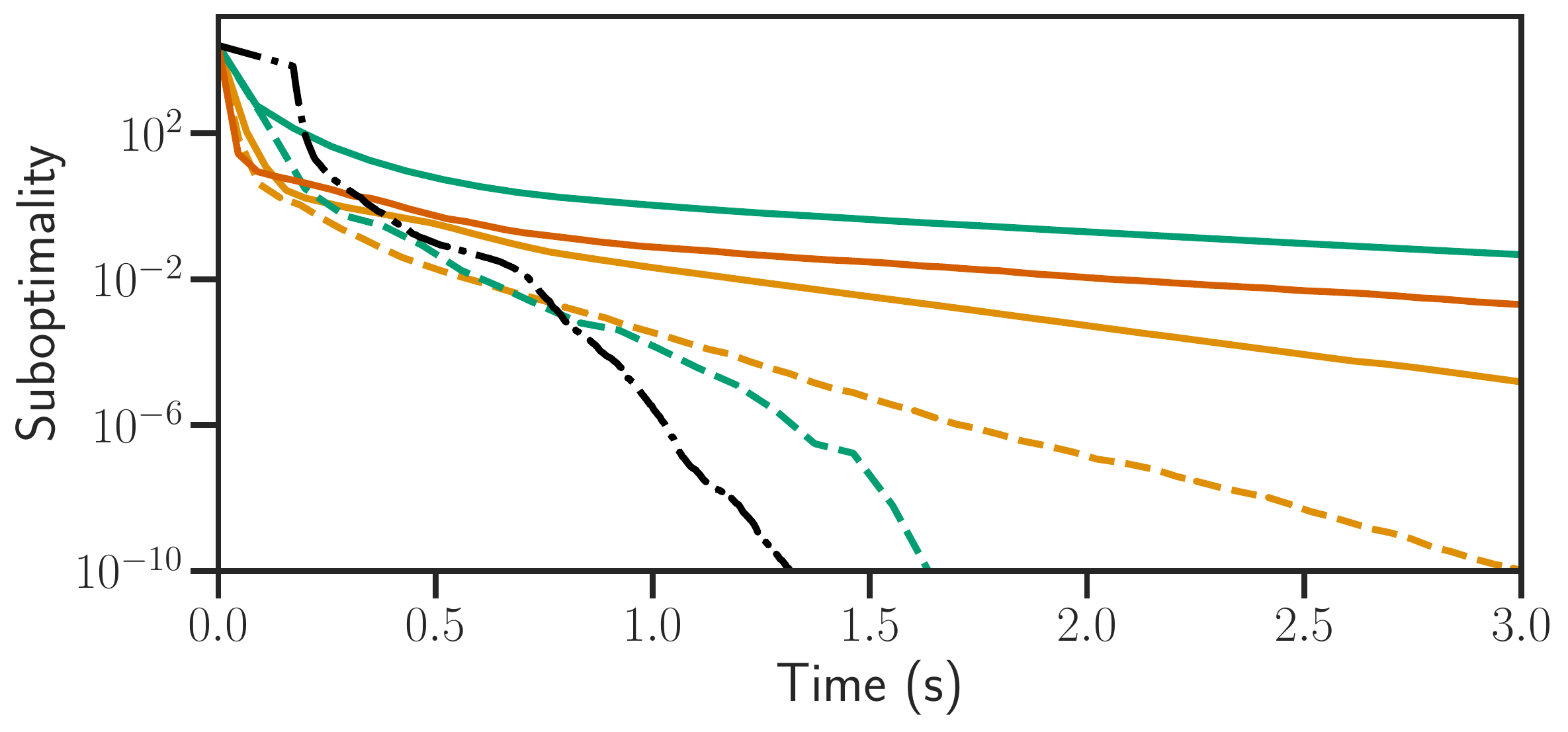}
\caption{\textbf{OLS, \emph{rcv1}.}
Suboptimality as a function of time on the $5000$ first columns of the dataset \emph{rcv1}.}
\label{fig:cd_sym_ols_rcv1}
\end{figure}

\Cref{fig:cd_sym_ols_rcv1} illustrates the convergence speed of cyclic and pseudo-symmetric coordinate descent on the \emph{rcv1} dataset.
Anderson acceleration provides speed up for both versions.
Interestingly, on this quadratic problem, the non extrapolated pseudo-symmetric iterations perform poorly, worse than cyclic coordinate descent.
However, the performances are reversed for their extrapolated counterparts: the pseudo-symmetrized version is better than the cyclic one (which has a nonsymmetric iteration matrix).
Finally, Anderson extrapolation on the pseudo-symmetrized version even reaches the conjugate gradient performance.
%
\subsection{Generalization to nonquadratic and proposed algorithm}
\label{sub:algorithm}
%
\begin{algorithm}[t]
    \caption{Online Anderson PCD (proposed)}
    \label{alg:anderson_cd}
    \setcounter{AlgoLine}{0}
    \SetKwInOut{Init}{init}
    \Init{$x^{(0)} \in \bbR^p$}
    \For{$k = 1, \ldots $}{
            $x = x^{(k-1)}$

            \For{$j = 1, \ldots p$}{
                $\tilde{x}_j = x_j$

                $x_j =
                \prox_{\frac{\lambda}{ L_j} g_j}
                (x_j - {A_{:j}^\top \nabla f(Ax)}/ {L_j} )$

                $Ax \pluseq (x_j - \tilde{x}_j) A_{:j}$
            }

            $x^{(k)} = x$  \tcp*[h]{regular iter. $\cO(np)$}

            \If(\tcp*[h]{extrapol., $\cO(K^3 + p K^2)$})
        {$k = 0 \quad \mathrm{mod} \, K$} {
            $U = [x^{(k - K + 1)} - x^{(k - K)}, \ldots, x^{(k)} - x^{(k - 1)}]$

            $c = (U^\top U)^{-1} \mathbf{1}_K / \mathbf{1}_K^\top (U^\top U)^{-1} \mathbf{1}_K \in \bbR^K$

            $x_{\mathrm{e}} = \sum_{i=1}^K c_i x^{(k - K + i)}$

            \If{$f(Ax_{\mathrm{e}}) + \lambda g(x_{\mathrm{e}}) \leq f(x^{(k)}) + \lambda g(x^{(k)})$}{
                $x^{(k)} = x_{\mathrm{e}}$\tcp*[r]{guaranteed convergence}
            }
        }
    }
    \Return{$x^{(k)}$}
\end{algorithm}
After devising and illustrating an Anderson extrapolated coordinate descent procedure for a simple quadratic objective, our goal is to apply Anderson acceleration on problems where coordinate descent achieve state-of-the-art results, \ie of the form:
\begin{align*}
    \min_{x \in \bbR^p}
    f(Ax) + \lambda g(x) \eqdef f(Ax) + \lambda \sum_{j=1}^p g_j(x_j)
    \enspace ,
\end{align*}
where $f: \bbR^n \to \bbR$ is convex, $\gamma$-smooth and $g_j$'s are proper, closed and convex functions
As examples, we allow $g = 0$, $g = \normin{x}_1$, $g = \tfrac12 \normin{x}_2^2$, $g = \normin{x}_1 + \frac{\rho}{2 \lambda} \normin{x}^2$.
In the nonquadratic case, for the proximal coordinate descent, following \citet{Klopfenstein2020}, the update of the $j$-st coordinate can be written $\psi_j : \bbR^p
\rightarrow \bbR^p$,
$x \mapsto x + (\prox_{\lambda g_{j} / L_{j}}
( x_{j} - \gamma_{j} \nabla_{j} f(x) ) - x_j )e_j $.
The nonlinear operator of one pass of coordinate descent (i.e. one pass on all the features) can thus be written: $\psi = \psi_p \circ \dots \circ \psi_1$.
One pass of proximal coordinate descent from $1$ to $p$ can be seen as a nonlinear fixed point iteration:
\begin{equation}
    x^{(k+1)} = \psi(x^{(k)})
    \enspace.
\end{equation}
In this case, $T$ is not a matrix, but a nonlinear operator.
However, as stated in \Cref{prop:dl_smooth_case}, asymptotically, this operator $T$ is linear.
\begin{proposition}\label{prop:dl_smooth_case}
    If $f$ is convex and smooth and $\mathcal{C}^2$, $g_j$ are convex smooth and $\mathcal{C}^2$, then
    $\psi$ is differentiable and
    \begin{align*}
        x^{(k+1)}
        = D \psi(x^*) (x^{(k)} - x^*) + x^*
        + o(\normin{x^{(k)} - x^{*}})
    \enspace .
    \end{align*}
\end{proposition}

Therefore, iterations of proximal coordinate descent for this problem lead to noisy linear iterations.
Proof of \Cref{prop:dl_smooth_case} can be found in \Cref{app:proofs}.

\Cref{fig:cd_sym_l2_logreg} shows the performance of Anderson extrapolation on a $\ell^2$-regularised logistic regression problem:
\begin{equation}
    \argmin_{x \in \bbR^p}
    \sum_{i=1}^n \log(1 + e^{- y_i A_{i:} x})
    + \frac{\lambda}{2} \norm{x}_2^2
    \enspace .
\end{equation}
One can see that despite the better theoretical properties of the pseudo-symmetrized coordinate descent, Anderson acceleration on coordinate descent seems to work better on the cyclic coordinate descent.
We thus choose to apply Anderson extrapolation on the cyclic coordinate descent (\Cref{alg:anderson_cd}), while adding a step checking the decrease of the objective function in order to ensure convergence.
\begin{figure}[tb]
    \centering
        \centering
        \includegraphics[height=23px]{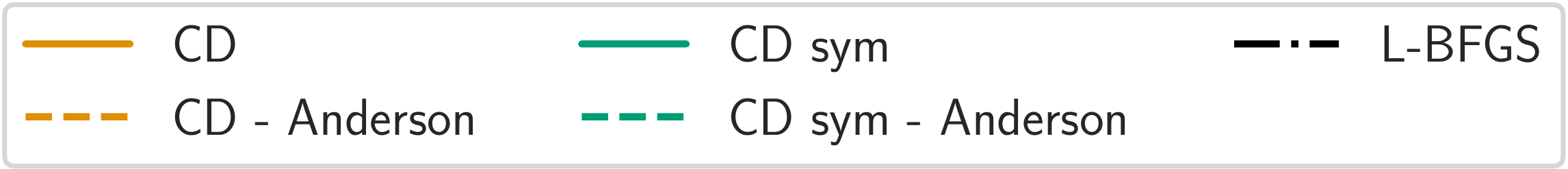}
        \includegraphics[width=\linewidth]{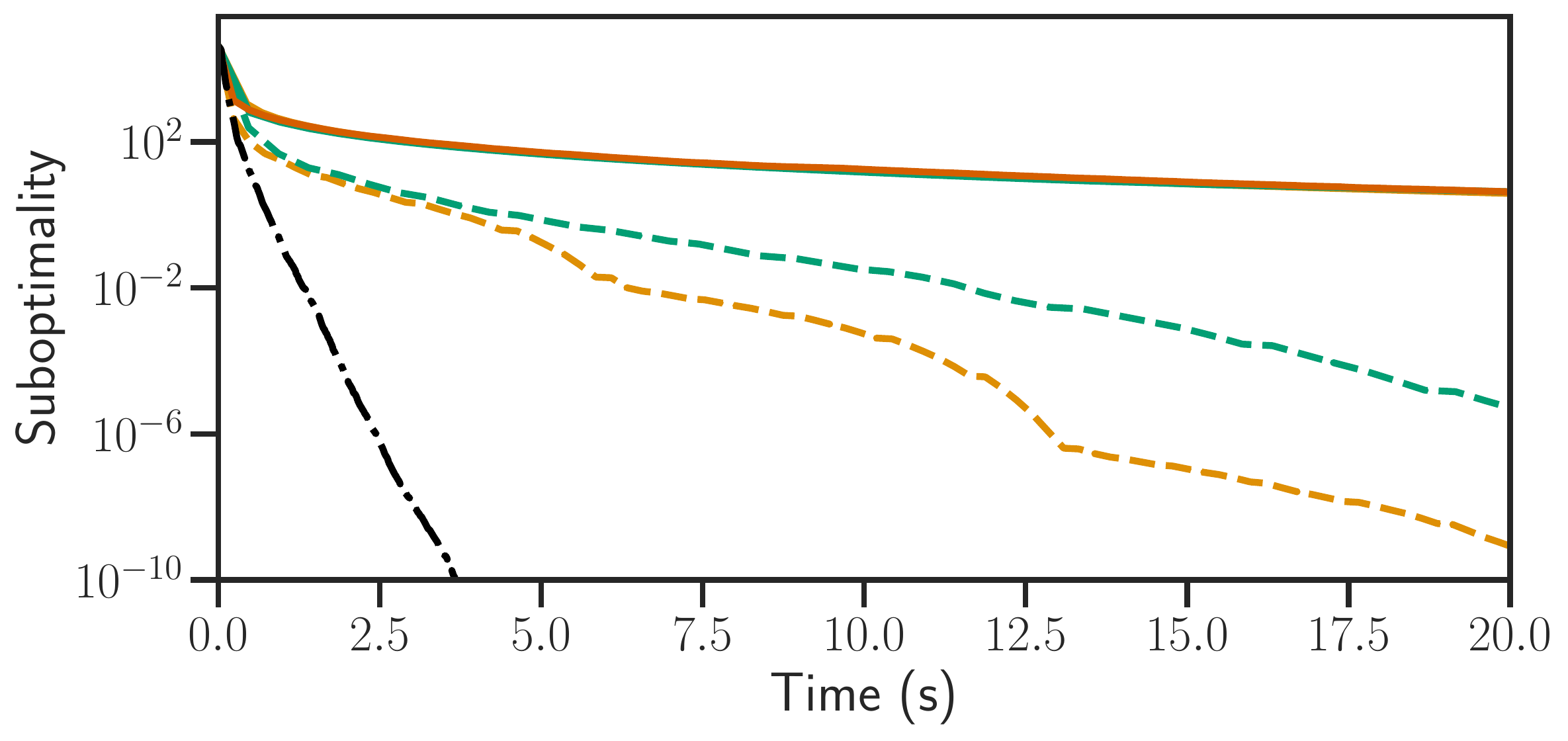}

    \caption{
        \textbf{$\ell_2$-regularised logistic regression, \emph{real-sim}.}
    Suboptimality as a function of time on the $2000$ first features of the \emph{real-sim} dataset, Tikhonov strength set so that $\kappa = 10^5$.
    }
    \label{fig:cd_sym_l2_logreg}
\end{figure}
Finally, we can also use \Cref{alg:anderson_cd} in the non smooth case where $g = \normin{\cdot}_1$,
since coordinate descent achieves support identification when the solution is unique, after which the objective becomes differentiable.
There is therefore a linear structure after a sufficient number of iterations \citep[Prop. 10]{Massias_Vaiter_Salmon_Gramfort2019}.


\section{Experiments}\label{sec:expes}
%
%
An implementation relying on numpy, numba and cython \citep{harris2020array,numba,cython}, with scripts to reproduce the figures, is available at \url{https://mathurinm.github.io/andersoncd}

We first show how we set the hyperparameters of Anderson extrapolation (\Cref{sub:parameter_setting}).
Then we show that Anderson extrapolation applied to proximal coordinate descent outperforms other first order algorithms on standard Machine Learning problems (\Cref{sub:numerical_comparison}).

\subsection{Parameter setting}\label{sub:parameter_setting}
%
Anderson extrapolation relies on $2$ hyperparameters: the number of extrapolated points $K$, and the amount of regularization eventually used when solving
the linear system to obtain the coefficients $c \in \bbR^K$.
Based on the conclusions of this section, we fix these parameters for all the subsequent experiments in \Cref{sub:numerical_comparison}: \emph{no regularization and $K=5$}.
%
%
\paragraph{Influence of the regularization.}
\begin{figure}[t]
    \centering
    \includegraphics[height=27px]{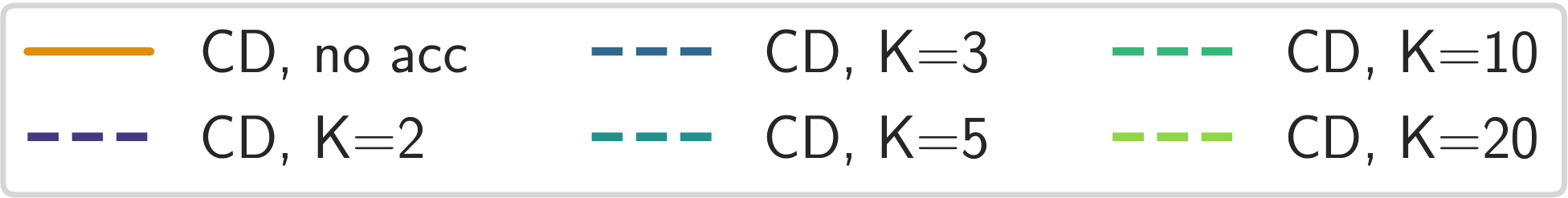}
    \includegraphics[height=110px]{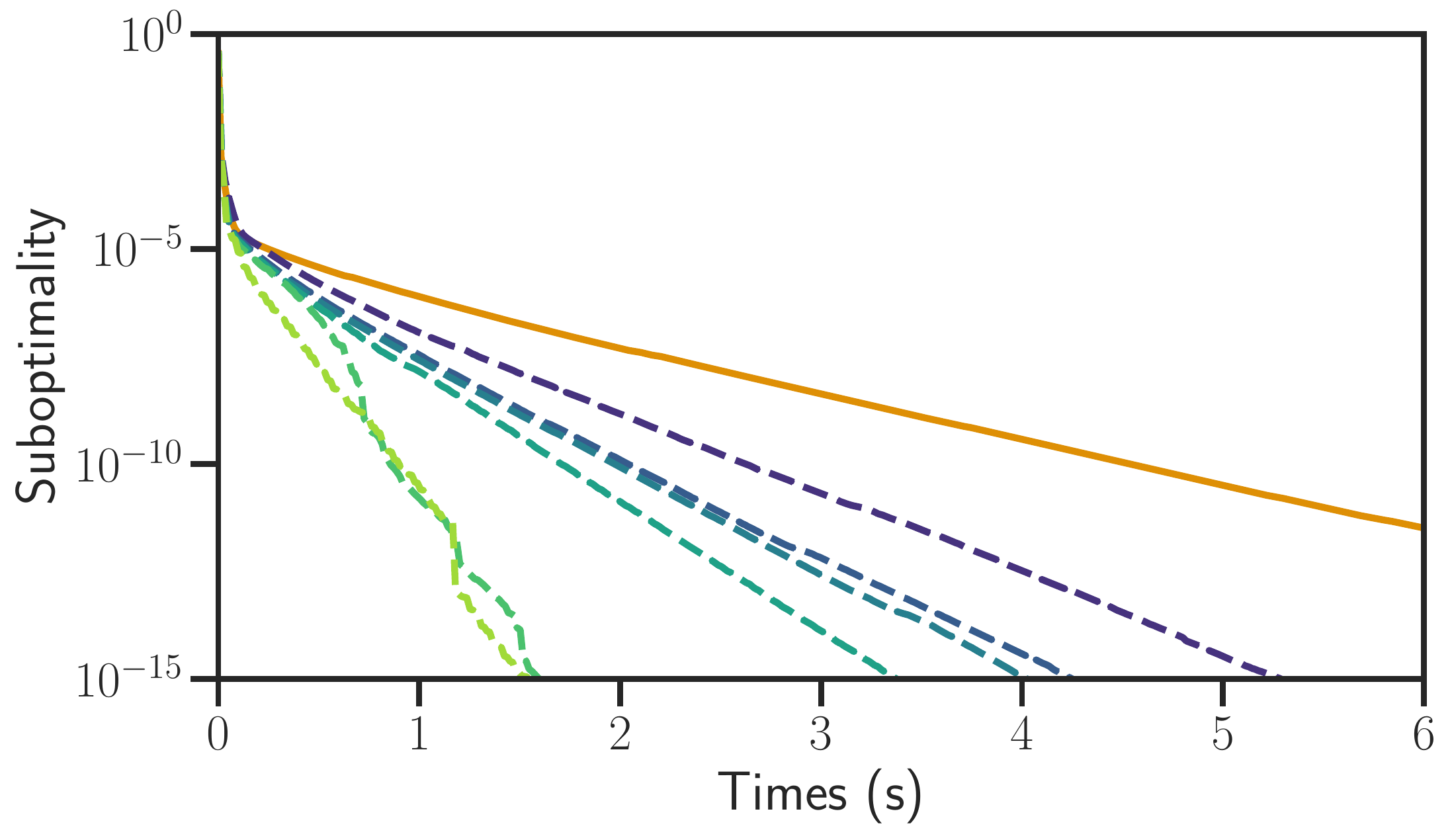}
    \caption{\textbf{Influence of $K$, quadratic, \emph{rcv1}.}
    Influence of the number of iterates $K$ used to perform Anderson extrapolation with coordinate descent (CD) on a quadratic with the \emph{rcv1} dataset ($2000$ first columns).}
    \label{fig:influenceK}
\end{figure}
\sloppy
\citet{Scieur_dAspremont_Bach2016} provided accelerated complexity rates for \emph{regularized} Anderson extrapolation: a term $\lambda_{\text{reg}} \norm{c}^2$ is added to the objective of \Cref{pb:anderson}.
The closed-form formula for the coefficients is then $(U^\top U + \lambda_\text{reg} \Id_K)^{-1} \mathbf{1}_K / \mathbf{1}_K^\top (U^\top U + \lambda_\text{reg} \Id_K)^{-1} \mathbf{1}_K$.

However, similarly to \cite{Mai_Johansson_19} and \citet{Poon_Liang2020} we observed that regularizing the linear system does not seem necessary, and can even hurt the convergence speed.
\Cref{fig:influence_reg_logreg} shows the influence of the regularization parameter on the convergence on the \emph{rcv1} dataset for a sparse logistic regression problem, with $K=5$ and  $\lambda = \lambda_{\max} / 30$.
The more the optimization problem is regularized, the more the convergence speed is deteriorated.
Thus \emph{we choose not to regularize} when solving the linear system for the extrapolation coefficients.
We simply check if the extrapolated point yields a lower objective function than the current regular iterate (see \Cref{alg:anderson_cd}).
%
\paragraph{Influence of $K$.}
%
\begin{figure}[tb]
    \centering
    \includegraphics[height=27px]{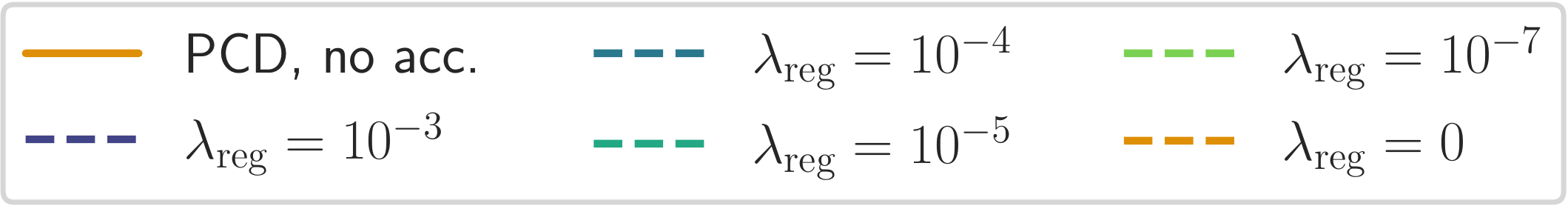}
    \includegraphics[height=110px]{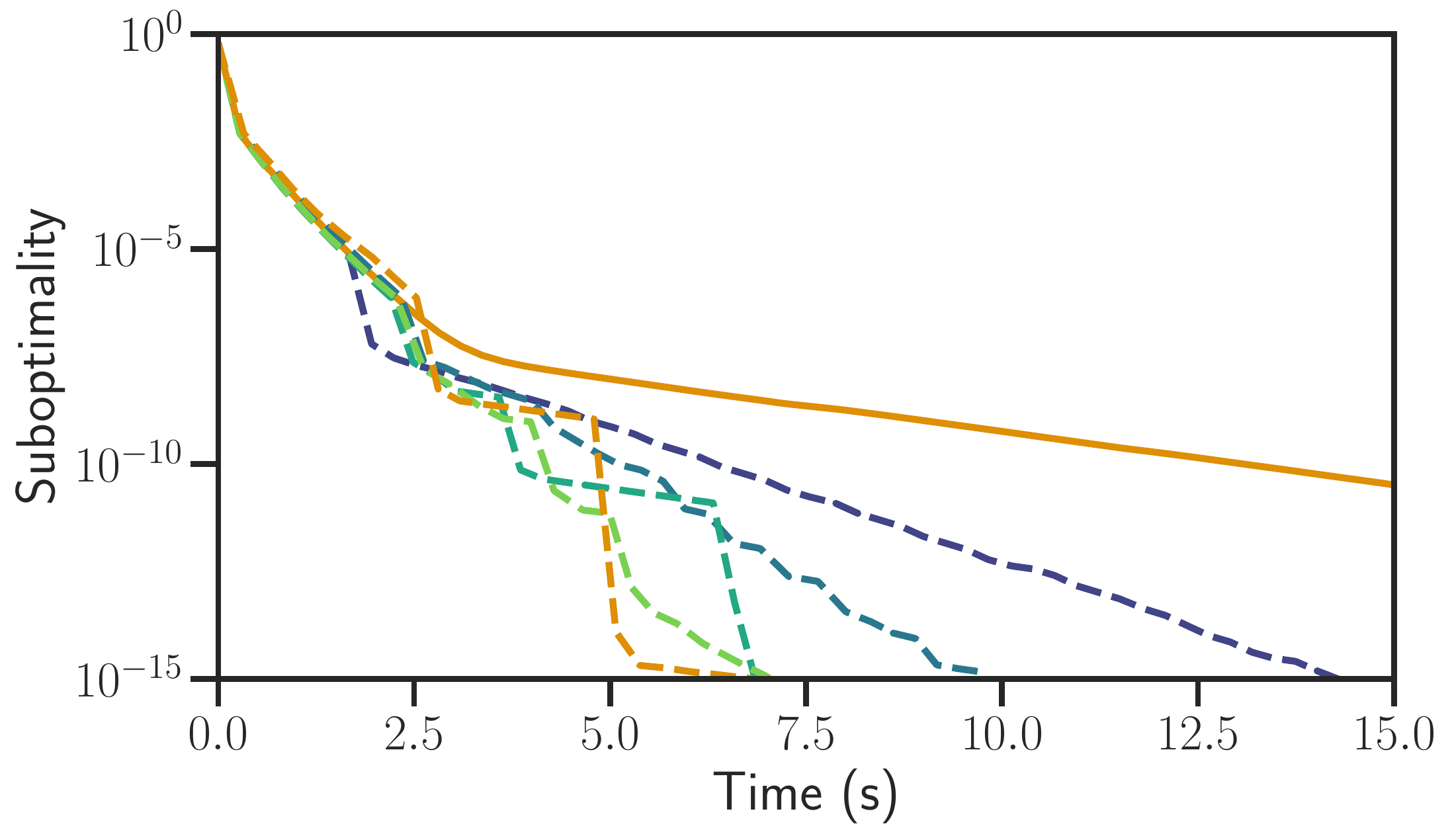}
    \caption{\textbf{Influence of
    $\lambda_{\mathrm{reg}}$, sparse logitic regression, \emph{rcv1}.}
    Influence of the regularization amount when solving a sparse logistic regression using Anderson extrapolation with proximal coordinate descent (PCD) on the \emph{rcv1} dataset, $K=5$, $\lambda = \lambda_{\max} / 30$.}
    \label{fig:influence_reg_logreg}
\end{figure}
%
%
\def\legendsize{0.7}
\def\widthlegend{0.75}
\begin{figure*}[tb]
    \centering
        \centering
        \includegraphics[width=0.7\linewidth]{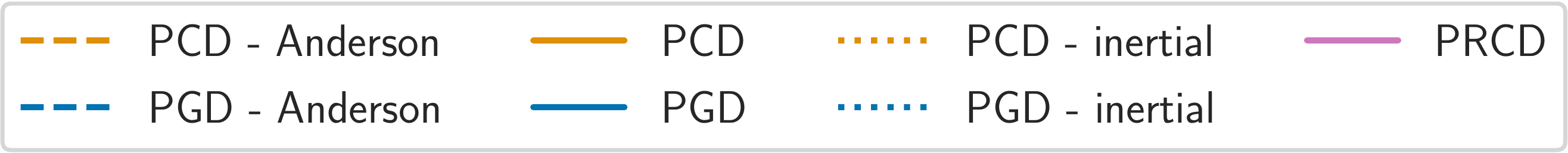}
        \includegraphics[width=1\linewidth]{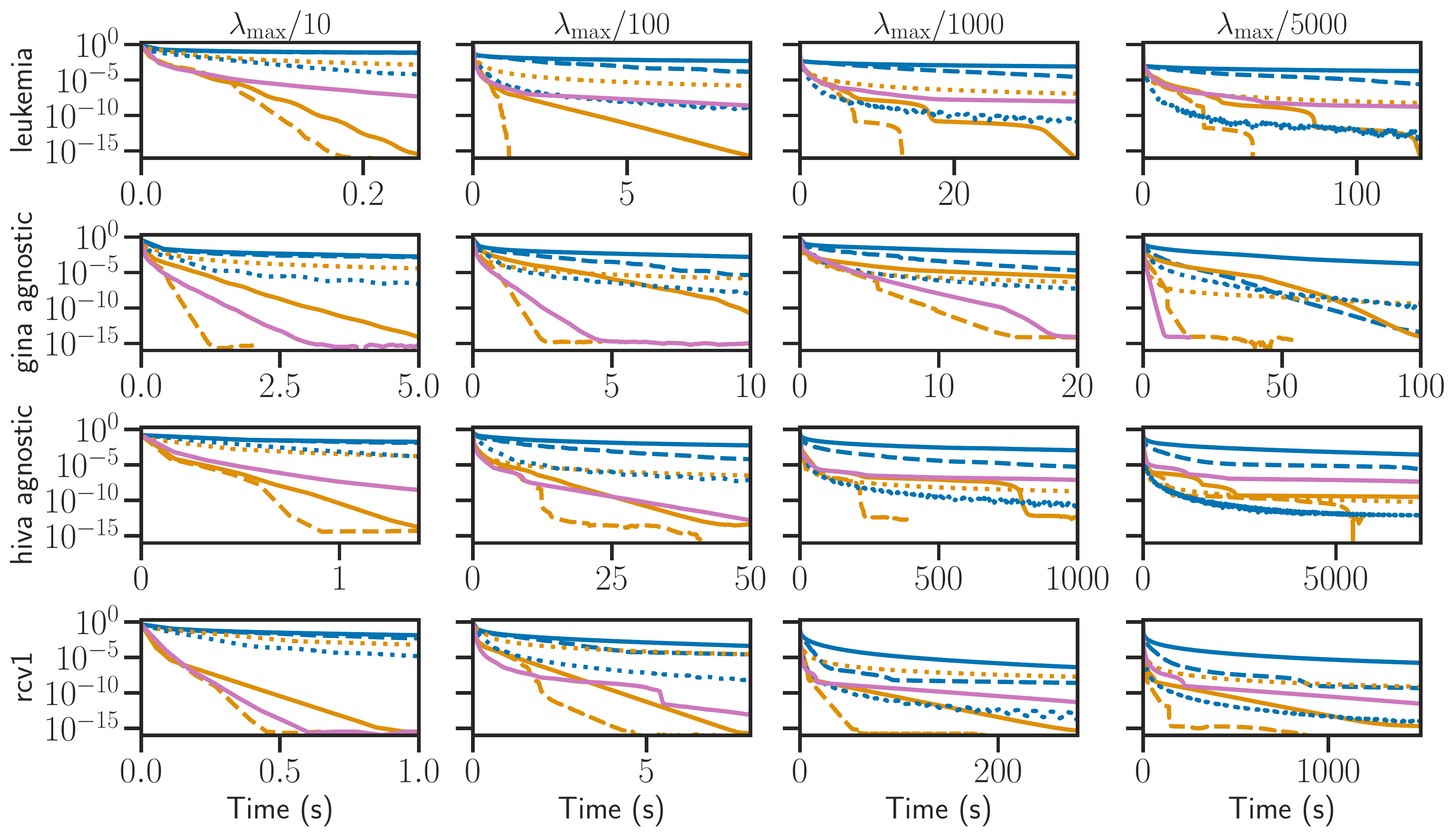}
    \caption{\textbf{Lasso, suboptimality.}
    Suboptimality as a function of time for the Lasso on multiple datasets and values of $\lambda$.}
    \label{fig:lasso_energies}
\end{figure*}
%
\Cref{fig:influenceK} shows the impact of $K$ on the convergence speed.
Although the performance depends on $K$, it seems that the dependency is loose, as for $K \in \{10, 20\}$ the acceleration is roughly the same.
Therefore, we do not treat $K$ as a parameter and fix it to $K=5$.
%
\paragraph{Computational overhead of Anderson extrapolation.}
With $\text{nnz}$ the number of nonzero coefficients, $K$ epochs (\ie $K$  updates of all coordinates) of CD \emph{without Anderson acceleration} cost:
\begin{align*}
    K \text{nnz}(A)
    \enspace .
\end{align*}
Every $K$ epochs, \Cref{alg:anderson_cd} requires to solve a $K \times K$ linear system.
Thus, $K$ epochs of CD \emph{with Anderson acceleration} cost:
\begin{align*}
    \underbrace{K \text{nnz}(A)}_{K \text{ passes of CD}} \hspace{1mm}
    + \hspace{3mm} \underbrace{K^2 \text{nnz}(w)}_{\text{form } U^\top U} \hspace{3mm}
    + \underbrace{K^3}_{\text{solve system}}
    \enspace .
\end{align*}
With our choice, $K=5$, the overhead of Anderson acceleration is marginal compared to a gradient call: $K^2 + K \text{nnz}(w) \ll \text{nnz}(A)$.
This can be observed in \Cref{fig:lasso_energies,fig:enet_energies,fig:logreg_energies}: even before acceleration actually occurs, Anderson PCD is not slower than regular PCD.
%
\def\widthfig{1}
\def\figsize{1}

\begin{figure*}[tb]
    \centering
        \centering
        \includegraphics[width=0.7\linewidth]{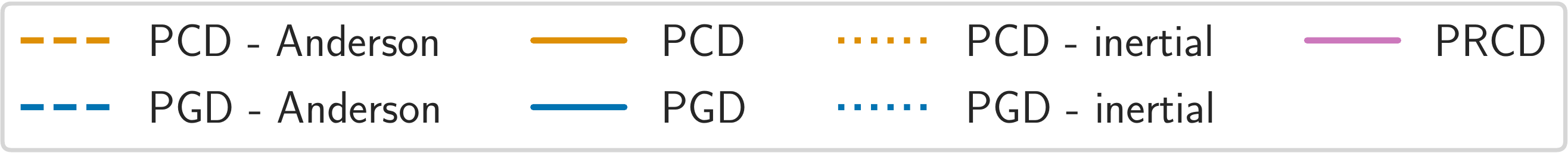}
        \includegraphics[width=\widthfig \linewidth]{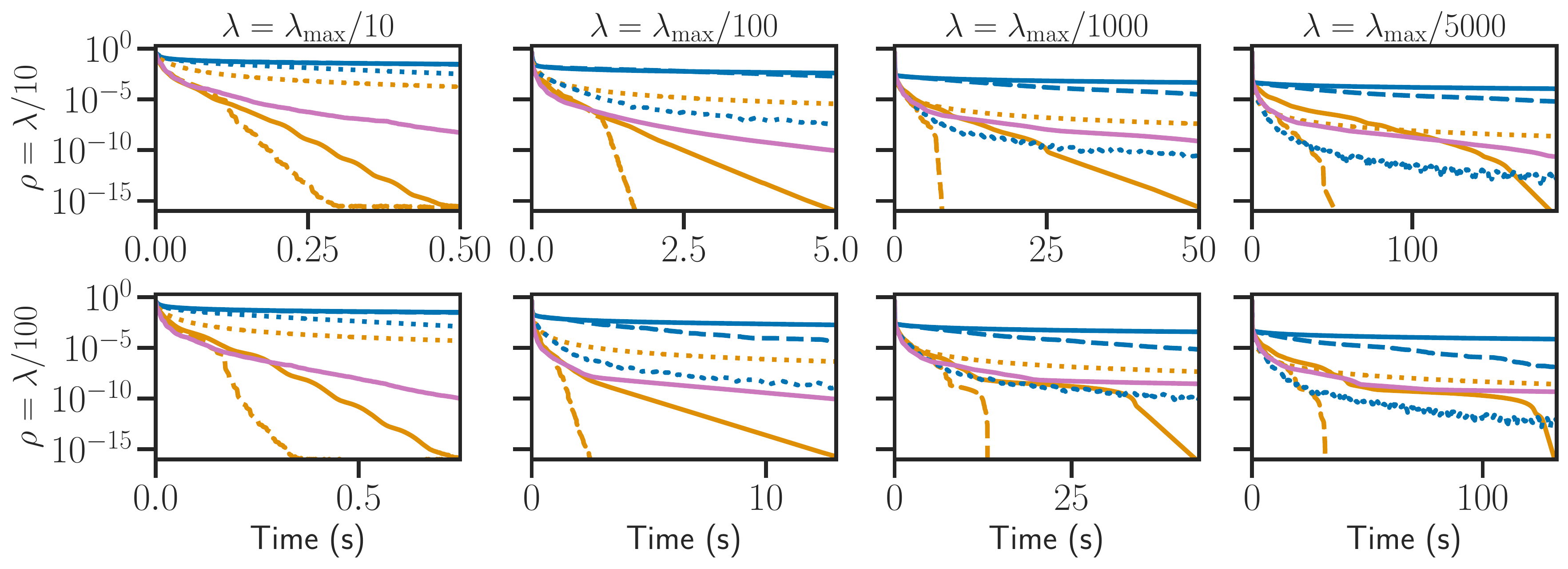}
    \caption{\textbf{Enet, suboptimality.}
    Suboptimality as a function of time for the elastic net on Leukemia dataset, for multiple values of $\lambda$ and $\rho$.
    }
    \label{fig:enet_energies}
\end{figure*}
%
\begin{figure*}[tb]
    \centering
        \centering
        \includegraphics[width=0.7\linewidth]{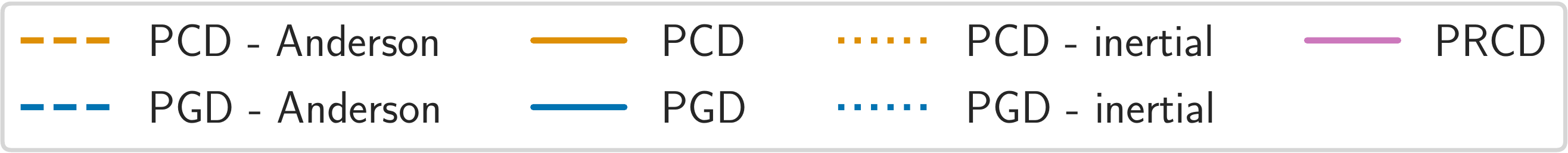}
        \includegraphics[width=\widthfig\linewidth]{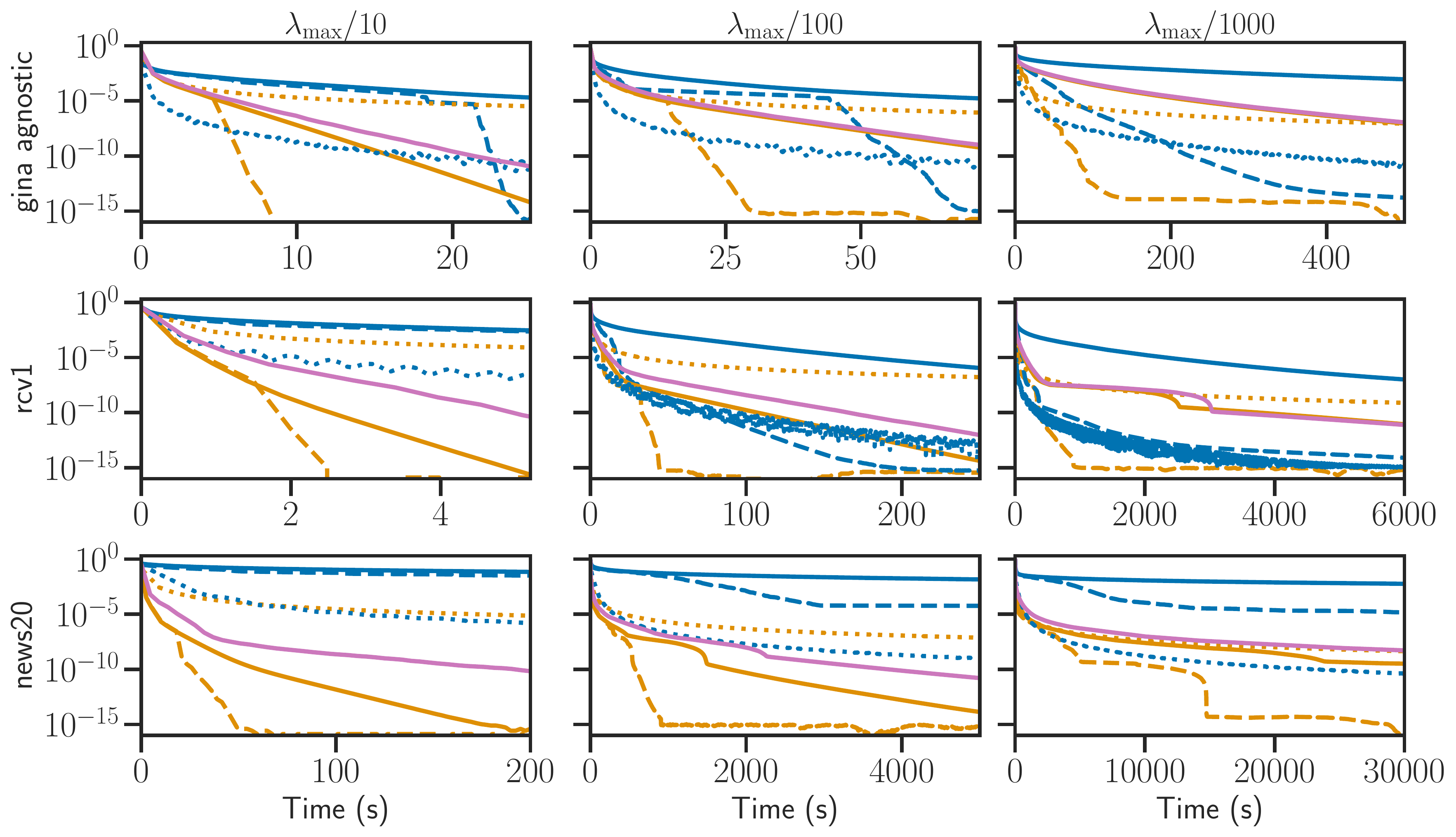}
    \caption{\textbf{$\ell_1$-regularised logistic regression, suboptimality.}
    Suboptimality as a function of time for $\ell_1$-regularized logistic regression on multiple datasets and values of $\lambda$.}
    \label{fig:logreg_energies}
  \end{figure*}
\subsection{Numerical comparison on Machine Learning problems}
\label{sub:numerical_comparison}
%
%
\sloppy
We compare multiple algorithms to solve popular Machine Learning problems: the Lasso, the elastic net, and sparse logistic regression (experiments on group Lasso are in \Cref{app:sub:group}).
The compared algorithms are the following:
\begin{itemize}
    \item Proximal gradient descent (PGD, \citealt{Combettes_Wajs2005}).
    \item Nesterov-like inertial PGD (FISTA, \citealt{Beck_Teboulle09}).
    \item Anderson accelerated PGD \citep{Mai_Johansson_19,Poon_Liang2020}.
    \item Proximal coordinate descent (PCD, \citealt{tseng2009coordinate}).
    \item Proximal coordinate descent with random index selection (PRCD, \citealt{Richtarik_Takac2014}).
    \item Inertial PCD \citep{Lin_Lu_Xiao_2014,Fercoq_Richtarik2015}.
    \item Anderson accelerated PCD (ours, \Cref{alg:anderson_cd}).
\end{itemize}
We use datasets from libsvm \citep{Fan_Chang_Hsieh_Wang_Lin08} and openml \citep{Feurer2019} (\Cref{tab:characteristics_datasets}), varying as much as possible to demonstrate the versatility of our approach.
We also vary the convergence metric: we use suboptimality in the main paper, while graphs measuring the duality gaps are in \Cref{app:additional_expe}.

\begin{table}[h!]
    \centering
    \caption{Datasets characteristics
    }
    \label{tab:characteristics_datasets}
    \begin{center}
    \begin{small}
    \begin{tabular}{lcccc}
        \toprule
        name   & $n$ & $p$ & density \\
        \midrule
        \emph{gina agnostic}      & $3468$    & $\num{970}$  & $1$   \\
        \emph{hiva agnostic}      & $4229$    & $\num{1617}$  & $1$   \\
        \emph{leukemia}      & $72$    & $\num{7129}$  & $1$   \\
        \emph{rcv1\_train} & $\num{20242}$ & $\num{19960}$  & $3.7 \, 10^{-3}$ \\
        \emph{real-sim} & $\num{72309}$ & $\num{20958}$  & $2.4 \, 10^{-3}$ \\
        \emph{news20}     & $\num{19996}$   & $\num{632983}$  & $6.1\, 10^{-4}$   \\
        \bottomrule
    \end{tabular}
    \end{small}
    \end{center}
\end{table}

\paragraph{Lasso.}
\Cref{fig:lasso_energies} shows the suboptimality $f(x^{(k)}) - f(x^*)$ of the algorithms on the Lasso problem:
\begin{equation}
    \argmin_{x \in \bbR^p}
    \frac{1}{2} \normin{y - A x}^2
    + \lambda \norm{x}_1
    \enspace ,
\end{equation}
as a function of time for multiple datasets and values of $\lambda$.
We parametrize $\lambda$ as a fraction of $\lambda_{\text{max}} = \normin{A^\top y}_\infty$, smallest regularization strength for which $x^* = 0$.
\Cref{fig:lasso_energies} highlights the superiority of proximal coordinate descent over proximal gradient descent for Lasso problems on real-world datasets, and the benefits of extrapolation for coordinate descent.
It shows that Anderson extrapolation can lead to a significant gain of performance.
In particular \Cref{fig:lasso_energies} shows that without restart, inertial coordinate descent \citep{Lin_Lu_Xiao_2014,Fercoq_Richtarik2015} can slow down the convergence, despite its  accelerated rate.
Note that the smaller the value of $\lambda$, the harder the optimization: when $\lambda$ decreases, more time is needed to reach a fixed suboptimality.
The smaller $\lambda$ is (\ie the harder the problem), the more efficient Anderson extrapolation is.

Other convergence metrics can be considered, since $f(x^*)$ is unknown to the practitioner: for the Lasso, it is also common to use the duality gap as a stopping criterion \citep{Massias_Gramfort_Salmon2018}.
Thus, for completeness, we provide \Cref{fig:lasso_gaps} in appendix, which shows the duality gap as a function of the number of iterations.
With this metric of convergence, Anderson PCD also
significantly
outperforms its competitors.
\paragraph{Elastic net.}
Anderson extrapolation is easy to extend to other estimators than the Lasso.
\Cref{fig:enet_energies} (and \Cref{fig:enet_gaps} in appendix) show the superiority of the Anderson extrapolation approach over proximal gradient descent and its accelerated version for the elastic net problem \citep{Zou_Hastie05}:
\begin{equation}
    \argmin_{x \in \bbR^p}
    \frac{1}{2n} \normin{y - A x}^2
    + \lambda \norm{x}_1
    + \frac{\rho}{2} \norm{x}_2^2
    \enspace .
\end{equation}
In particular, we observe that the more difficult the problem, the more useful the Anderson extrapolation: it is visible on \Cref{fig:enet_gaps,fig:enet_energies} that going from $\rho = \lambda/10$ to $\rho = \lambda/100$ lead to an increase in time to achieve similar suboptimality for the classical proximal coordinate descent, whereas the impact is more limited on the coordinate descent with Anderson extrapolation.

Finally, for a nonquadratic data-fit, here sparse logistic regression, we still demonstrate the applicability of extrapolated coordinate descent.
%
\paragraph{Sparse logistic regression.}
%
\Cref{fig:logreg_energies} represents the suboptimality as a function of time on a sparse logistic regression problem:
\begin{equation}
    \argmin_{x \in \bbR^p}
    \sum_{i=1}^n \log(1 + e^{- y_i A_{i:} x})
    + \lambda \norm{x}_1
    \enspace ,
\end{equation}
for multiple datasets and values of $\lambda$.
We parametrize $\lambda$ as a fraction of $\lambda_{\text{max}} = \normin{A^\top y}_\infty / 2$.
As for the Lasso and the elastic net, the smaller the value of $\lambda$, the harder the problem and Anderson CD outperforms its competitors.

\paragraph{Conclusion}
In this work, we have proposed to accelerate coordinate descent using Anderson extrapolation.
We have exploited the fixed point iterations followed by coordinate descent iterates on multiple Machine Learning problems to improve their convergence speed.
We have circumvented the non-symmetricity of the iteration matrices by proposing a pseudo-symmetric version for which accelerated convergence rates have been derived.
In practice, we have performed an extensive validation to demonstrate large benefits on multiple datasets and problems of interests.
For future works, the excellent performance of Anderson extrapolation for cyclic coordinate descent calls for a more refined analysis of the known bounds, through a better analysis of the spectrum and numerical range of the iteration matrices.

\section*{Acknowledgements}
Part of this work has been carried out at the Machine Learning Genoa (MaLGa) center, Universit\`a di Genova (IT).
M. M. acknowledges the financial support of the European
Research Council (grant SLING 819789).
This work was partially funded by the ERC Starting Grant SLAB ERC-StG-676943.
The authors thank Louis Béthune and Mathieu Blondel for pointing out the difference between the proposed \emph{online} algorithm and previous \emph{online} algorithms in the literature.

\clearpage

\bibliographystyle{plainnat}
\bibliography{biblio_extracd}

\clearpage
\onecolumn
\appendix


\section{Additional experiments}
\label{app:additional_expe}

\subsection{Gaps as a function of time}
\label{app:subsec_gaps}

In this section, we include the counterparts of \Cref{fig:lasso_energies,fig:enet_energies,fig:logreg_energies}, but display the duality gap instead of the suboptimality.
Indeed, since $x^*$ in not available in practice, the suboptimality cannot be used as a stopping criterion.
To create a dual feasible point, we use the classical technique of residual rescaling \citep{Mairal}.

\def\widthlegend{0.9}
\def\widthimag{0.9}

\def\widthlegend{0.7}
\def\widthimag{0.9}
%
\begin{figure}[H]
    \centering
        \centering
        \includegraphics[width=\widthlegend \linewidth]{energies_time_legend}
        \includegraphics[width=\widthimag \linewidth]{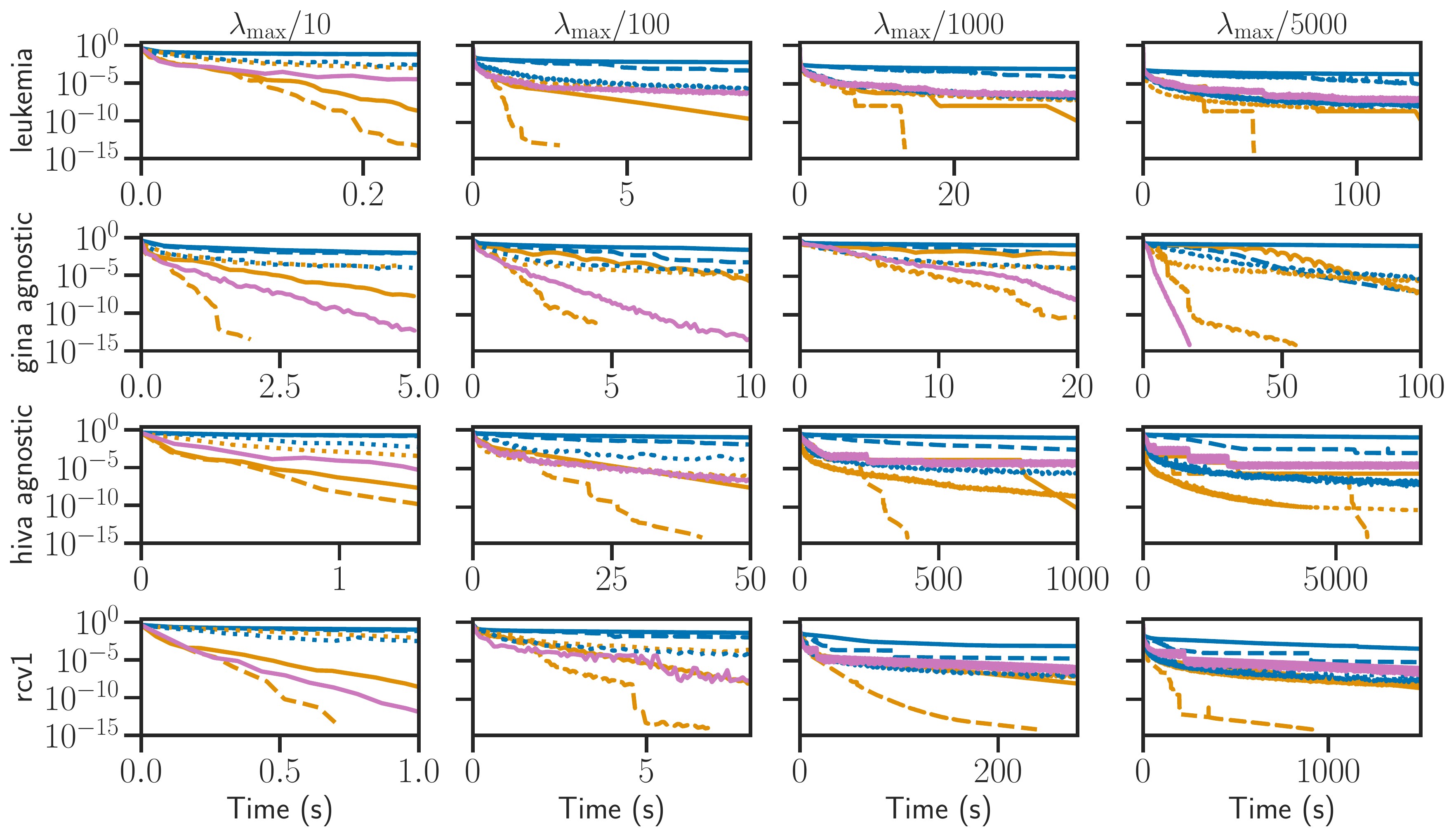}
    \caption{\textbf{Lasso, duality gap.} Duality gap along time for the Lasso on various datasets and values of $\lambda$.
    }
    \label{fig:lasso_gaps}
\end{figure}
%
%
\begin{figure}[H]
    \centering
        \centering
        \includegraphics[width=\widthlegend \linewidth]{energies_enet_time_legend}
        \includegraphics[width=\widthimag  \linewidth]{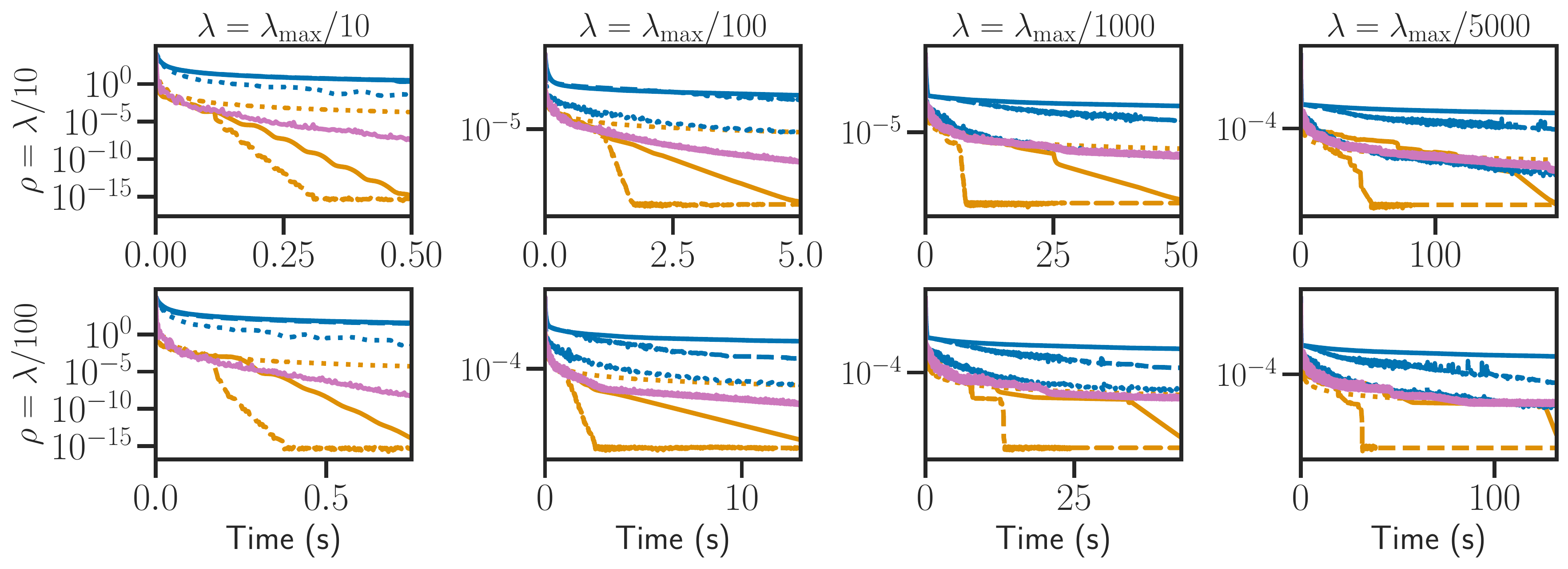}
    \caption{\textbf{Elastic net, duality gap.}
    Duality gap as a function of time for the elastic net on Leukemia dataset, for multiple values of $\lambda$ and $\rho$.
    }
    \label{fig:enet_gaps}
\end{figure}
%
%
\begin{figure}[H]
    \centering
    \centering
    \includegraphics[width=\widthlegend \linewidth]{energies_logreg_time_legend}
    \includegraphics[width=0.8 \linewidth]{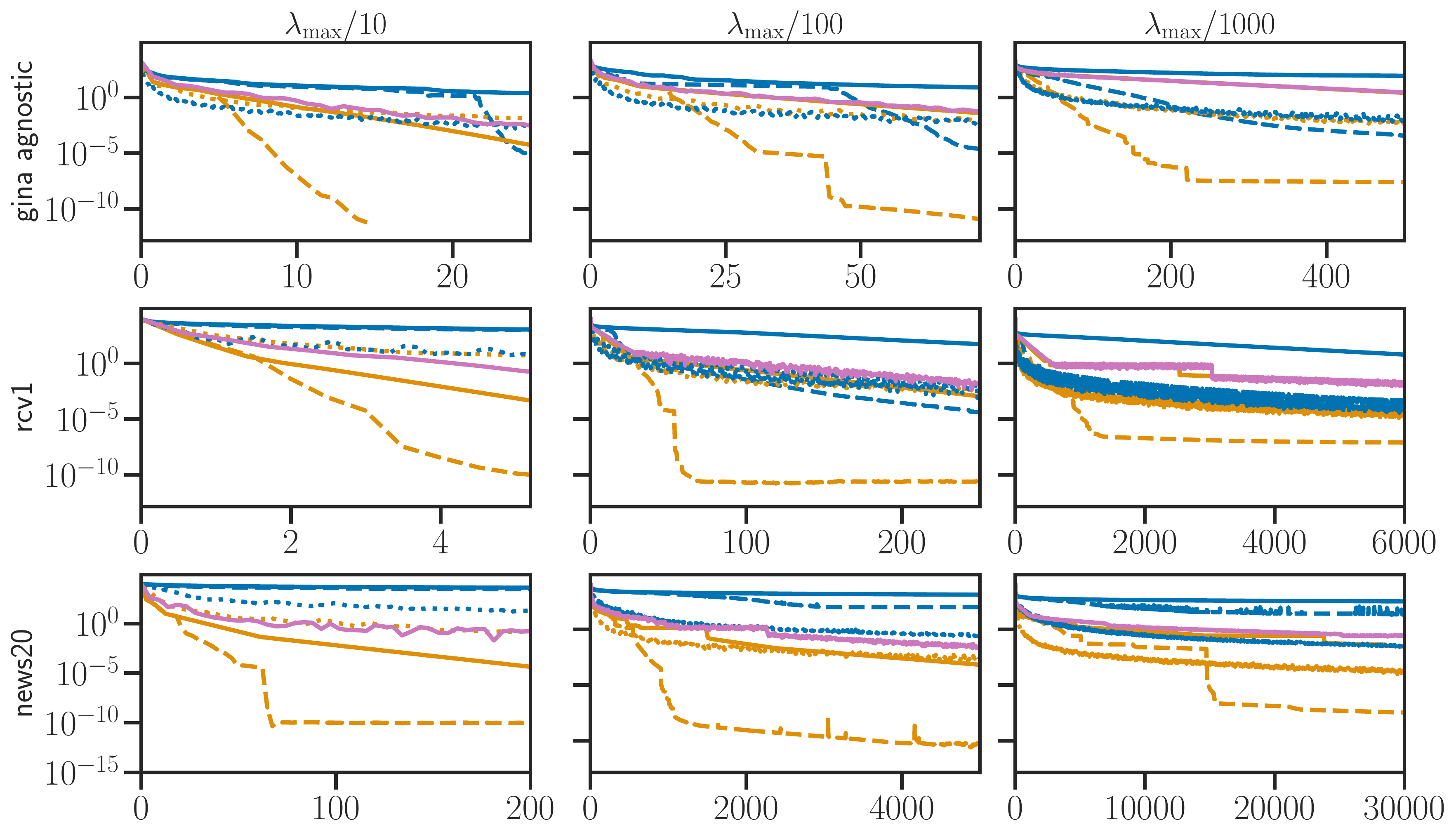}
    \caption{\textbf{$\ell_1$-regularised logistic regression, duality gap.} Duality gap as a function of time for $\ell_1$-regularized logistic regression on multiple datasets and values of $\lambda$.}
    \label{fig:logreg_dual_gaps}
\end{figure}

\subsection{Group Lasso}\label{app:sub:group}
%
In this section we consider the group Lasso, with a design matrix $A \in \bbR^{n \times p}$, a target $y \in \bbR^n$, and a partition $\cG$ of $[p]$ (elements of the partition being the disjoints groups):
\begin{align}
    \argmin_{x \in \bbR^p}
    \frac{1}{2}\norm{y - A x}^2
    + \lambda \sum_{g \in \cG} \norm{x_g}
    \enspace ,
\end{align}
where for $g \in \cG$, $x_g \in \bbR^{|g|}$ is the subvector of $x$ composed of coordinates in $g$.
the group Lasso can be solved via proximal gradient descent and by block coordinate descent (BCD), the latter being amenable to Anderson acceleration.
As \Cref{fig:group_subopt} shows, the superiority of Anderson accelerated block coordinate descent is on par with the one observed on the problems studied above.
\begin{figure}[H]
    \centering
    \centering
    \includegraphics[width=0.6 \linewidth]{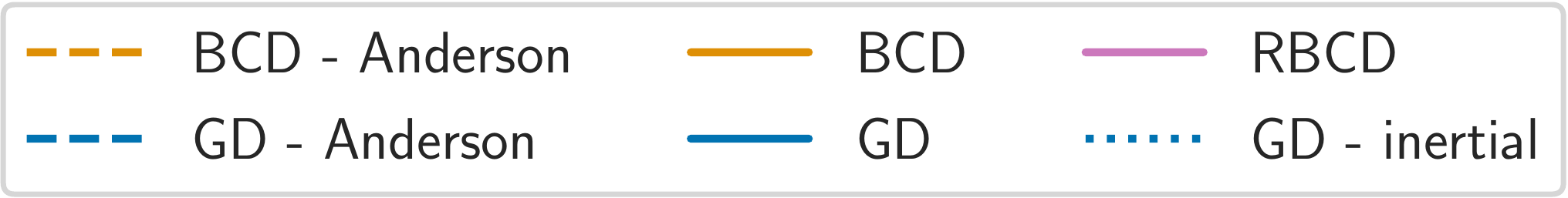}
    \includegraphics[width=0.6 \linewidth]{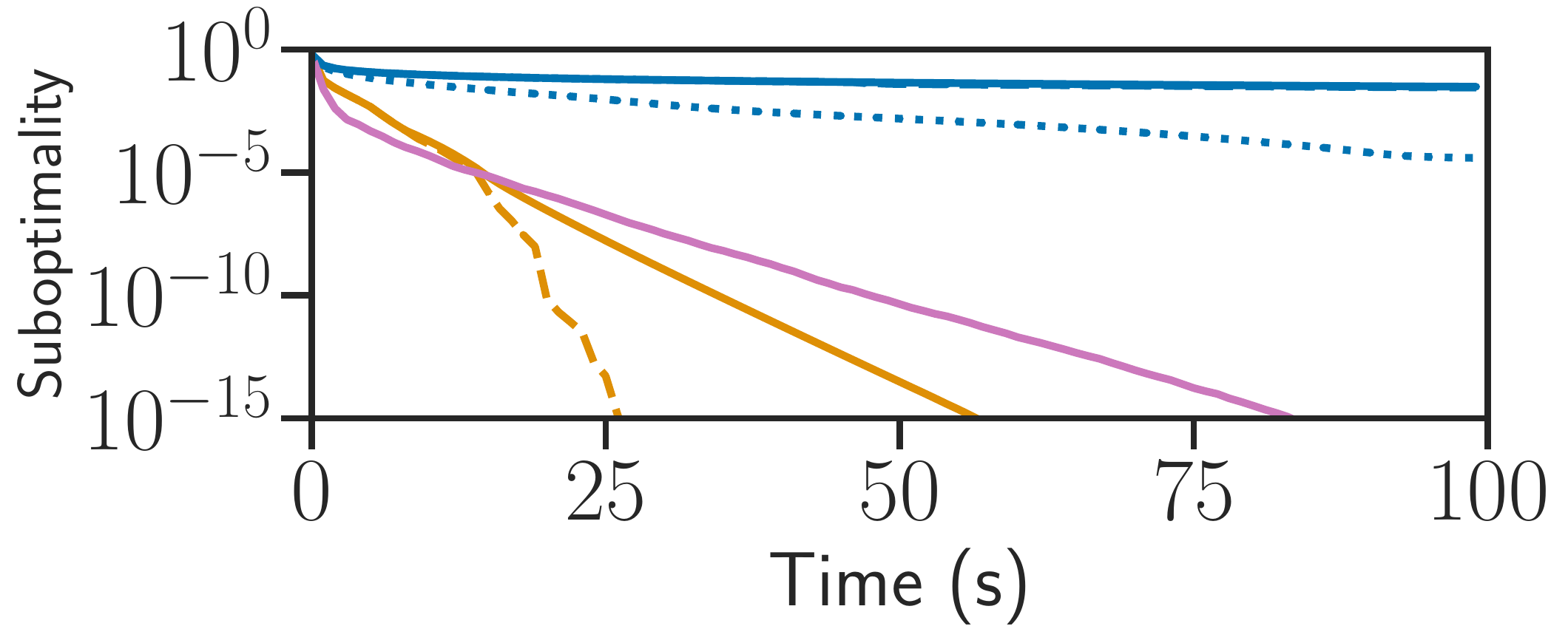}
    \caption{\textbf{Group Lasso, suboptimality.}
    Suboptimality as a function of time for the group Lasso on the \emph{Leukemia} dataset, $\lambda = \lambda_{\max} / 100$. Groups are artificially taken as consecutive blocks of 5 features.}
    \label{fig:group_subopt}
\end{figure}
%

\section{Proofs of \Cref{prop:acc_rate_cdsym,prop:dl_smooth_case}}
\label{app:proofs}
%

\subsection{Proofs of \Cref{prop:acc_rate_cdsym}}
%
%
\begin{lemma}\label{lem:link_xe_res}
    First we link the quantity computed in \Cref{pb:anderson} to the extrapolated quantity $ \sum_{i=1}^k c_i x^{(i - 1)}$.
    For all $c \in \bbR^k$ such that $\sum_{i=1}^k c_i = 1$:
    \begin{equation}
        \sum_{i=1}^k c_i ( x^{(i)} - x^{(i-1)} )
        =
        (T - \Id ) \left ( \sum_{i=1}^k c_i x^{(i-1)} - x^* \right )
        \enspace.
    \end{equation}
\end{lemma}
\begin{proof}
    Since $x^{(i)} = T x^{(i-1)} + (x^* - T x^*)$,
    \begin{align}
        c_i ( x^{(i)} - x^{(i-1)})
        &=
        c_i  (T x^{(i-1)} + x^* - T x^* - x^{(i-1)})
        \nonumber \\
        &=
        (T - \Id) c_i (x^{(i-1)} - x^*) \enspace.
    \end{align}
    Hence, since $\sum_1^k c_i = 1$,
    \begin{equation}
        \sum_{i=1}^k c_i ( x^{(i)} - x^{(i-1)} )
        =
        (T - \Id) \left ( \sum_{i=1}^k c_i x^{(i-1)} - x^* \right )
        \enspace .
    \end{equation}
\end{proof}
\begin{lemma}\label{lem:link_xe_x0}
    For all $c \in \bbR^k$ such that $\sum_{i=1}^k c_i = 1 $,
    \begin{align}
        \normin{(T - \Id)(x_\text{e-off}^{(k)} - x^{*} )}
        \leq
        \sqrt{\kappa(H)} \Big\Vert \sum_{i=0}^{k-1} c_i S^i \Big\Vert
        \normin{(T - \Id) (x^{(0)} - x^{*})}
        \enspace .
    \end{align}
\end{lemma}
\begin{proof}
    In this proof, we denote by $c^*$ the solution of \eqref{pb:anderson}.
    We use the fact that for all $c \in \bbR^k$ such that $\sum_{i=1}^k c_i = 1 $,
    \begin{align}
        \normin{\sum_{i=1}^k c_i^* ( x^{(i)} - x^{(i-1)} )}
        =
        \min_{\substack{
            c \in \bbR^k \\
            \sum_i c_i =1}}
            \normin{\sum_{i=1}^k c_i ( x^{(i)} - x^{(i-1)} )}
        \leq
        \normin{\sum_{i=1}^k c_i ( x^{(i)} - x^{(i-1)} )} \enspace.
        \label{eq:minimality_ci}
    \end{align}
    Then we use twice \Cref{lem:link_xe_res} for the left-hand and right-hand side of \Cref{eq:minimality_ci}.
    Using \Cref{lem:link_xe_res} with the $c_i^*$ minimizing \Cref{pb:anderson}
    we have for all $c_i \in \bbR$ such that $\sum_{i=1}^k c_i = 1$ :
    \begin{align}
        \normin{(T - \Id) (x_\text{e} - x^{*} )}
        &=
        \normin{\sum_{i=1}^k c_i^* ( x^{(i)} - x^{(i-1)} )}
        \nonumber
        \\
        &\leq
        \normin{\sum_{i=1}^k c_i ( x^{(i)} - x^{(i-1)} )}
        \nonumber
        \\
        &=
        \normin{(T - \Id) \sum_{i=1}^k c_i ( x^{(i-1)} - x^{(*)} )}
        \nonumber \\
        &=
        \normin{(T - \Id)  \sum_{i=1}^k c_i T^{i-1} (x^{(0)} - x^*)}
        \nonumber \\
        &=
        \normin{\sum_{i=1}^k c_i T^{i-1} (T - \Id) (x^{(0)} - x^*)}
        \nonumber \\
        &=
        \normin{\sum_{i=1}^k c_i T^{i-1}} \times \normin{(T - \Id ) (x^{(0)} - x^*)}
        \nonumber \\
        &\leq
        \normin{H^{-1/2} \sum_{i=1}^k c_i S^{i-1} H^{1/2}}
        \times \normin{(T - \Id )(x^{(0)} - x^*)}
        \nonumber \\
        &\leq
        \sqrt{\kappa(H)} \normin{\sum_{i=1}^k c_i S^{i-1}}
        \times  \normin{(T - \Id )(x^{(0)} - x^*)}
        \enspace .
    \end{align}
\end{proof}

\begin{proof}
    We apply \Cref{lem:link_xe_x0} by choosing $c_i$ equal to the Chebyshev weights $c_i^\text{Cb}$.
    Using the proof of \citet[Prop. B. 2]{Barre2020}, we have, with $ \zeta = \frac{1 - \sqrt{1 - \rho(T)}}{1 + \sqrt{1- \rho(T)}}$:
    \begin{equation}
        \normin{\sum_{i=1}^k c_i^\text{Cb} S^{i-1}}
        \leq
        \tfrac{2 \zeta^{k-1}}{1 + \zeta^{2(k-1)}}
        \enspace.
    \end{equation}
    Combined with \Cref{lem:link_xe_x0} this concludes the proof:
    \begin{align}
        \normin{(T -  \Id) (x_\text{e} - x^{*} )}
        &\leq
        \sqrt{\kappa(H)}
        \normin{\sum_{i=1}^k c_i S^{i-1}}
        \normin{(T -  \Id) (x^{(0)} - x^{*} )}
        \\
        &\leq
        \sqrt{\kappa(H)}
        \tfrac{2\zeta^{k-1}}{1 + \zeta^{2(k-1)}}
        \normin{(T -  \Id) (x^{(0)} - x^{*} )}
        \enspace .
    \end{align}
\end{proof}

\subsection{Proof of \Cref{prop:dl_smooth_case}}
%
Since $g_j$ are $\mathcal{C}^2$ then $\prox_{g_j}$ are $\mathcal{C}^1$, see \citet[Cor. 1.b]{Gribonval_Nikolova2020}.
Moreover, $f$ is $\mathcal{C}^2$ and following \citet{Massias_Vaiter_Salmon_Gramfort2019,Klopfenstein2020} we have that:
\begin{align}
    \psi_j : \bbR^p
    &\rightarrow \bbR^p
    \nonumber \\
    x
    &\mapsto \prox_{g_j}
    \left ( \begin{array}{c}
        x_1\\
        \vdots \\
        x_{j-1} \\
        \prox_{\lambda g_j / L_j}
        \big( x_j - \frac{1}{L_j}\nabla_j f(x) \big)\\
        x_{j+1} \\
        \vdots \\
        x_p
    \end{array} \right)
    \enspace ,
\end{align}
is differentiable.
Thus we have that the fixed point operator of coordinate descent:
$\psi = \psi_p \circ \dots \circ \psi_1$ is differentiable.
\Cref{prop:dl_smooth_case} follows from the Taylor expansion of $\psi$ in $x^*$.

\end{document}